\documentclass[12pt]{article}
\usepackage[top=1.0in, bottom=1.0in, left=1.12in, right=1.12in]{geometry}

\usepackage{hyperref}
\usepackage{amsmath, amssymb, amsthm}
\usepackage{color}
\usepackage{algpseudocode}
\usepackage{enumerate}
\usepackage{wrapfig}
\usepackage{graphicx} 
\usepackage{caption}
\usepackage{subcaption}
\usepackage{algorithm,algcompatible}
\usepackage{bm}
\usepackage{hyperref}
\usepackage{amsmath, amssymb, amsthm}
\usepackage{natbib}
\usepackage{graphicx}
\usepackage{soul}

\newcommand{\beq}{\vspace{0mm}\begin{equation}}
\newcommand{\eeq}{\vspace{0mm}\end{equation}}
\newcommand{\beqs}{\vspace{0mm}\begin{eqnarray}}
\newcommand{\eeqs}{\vspace{0mm}\end{eqnarray}}
\newcommand{\barr}{\begin{array}}
\newcommand{\earr}{\end{array}}

\newcommand{\wv}{\boldsymbol{w}}
\newcommand{\xv}{\boldsymbol{x}}

\newcommand{\zv}{\boldsymbol{z}}

\newcommand{\given}{\,|\,}

\newtheorem{thm}{Theorem} 
\newtheorem{lem}[thm]{Lemma}

\usepackage{booktabs}
\usepackage{array}
\newcolumntype{L}[1]{>{\raggedright\let\newline\\\arraybackslash\hspace{0pt}}m{#1}}
\newcolumntype{C}[1]{>{\centering\let\newline\\\arraybackslash\hspace{0pt}}m{#1}}
\newcolumntype{R}[1]{>{\raggedleft\let\newline\\\arraybackslash\hspace{0pt}}m{#1}}

\usepackage{changes}

\newtheorem{property}{Property}

\newcolumntype{P}[1]{>{\centering\arraybackslash}p{#1}}

\algblock{ParFor}{EndParFor}
\algnewcommand\algorithmicparfor{\textbf{parfor}}
\algnewcommand\algorithmicpardo{\textbf{do}}
\algnewcommand\algorithmicendparfor{\textbf{end\ parfor}}
\algrenewtext{ParFor}[1]{\algorithmicparfor\ #1\ \algorithmicpardo}
\algrenewtext{EndParFor}{\algorithmicendparfor}
\algnewcommand\INPUT{\item[\textbf{Input:}]}%
\algnewcommand\OUTPUT{\item[\textbf{Output:}]}%

\usepackage{setspace}
\title{
MCMC-Interactive Variational Inference
} 
\author{Quan Zhang\footnote{
\texttt{\small{quan.zhang@broad.msu.edu}}, Broad College of Business, Michigan State University.}, 
Huangjie Zheng\footnote{
\texttt{\small{huangjie.zheng@utexas.edu}}, the University of Texas at Austin}, 
and Mingyuan Zhou\footnote{
\texttt{\small{mingyuan.zhou@mccombs.utexas.edu}}, McCombs School of Business, the University of Texas at Austin.}
}

\date{\today}
\begin{document}
\begin{spacing}{1.25}
\maketitle

\begin{abstract}
Leveraging well-established MCMC strategies, we propose MCMC-interactive variational inference (MIVI) to not only estimate the posterior in a time constrained manner, but also facilitate the design of MCMC transitions. Constructing a variational distribution followed by a short Markov chain that has parameters to learn, MIVI takes advantage of the complementary properties of variational inference and MCMC to encourage mutual improvement. On one hand, with the variational distribution locating high posterior density regions, the Markov chain is optimized within the variational inference framework to efficiently target the posterior despite a small number of transitions. On the other hand, the optimized Markov chain with considerable flexibility guides the variational distribution towards the posterior and alleviates its underestimation of uncertainty. Furthermore, we prove the optimized Markov chain in MIVI admits extrapolation, which means its marginal  distribution gets closer to the true posterior as the chain grows. Therefore, the Markov chain can be used separately as an efficient MCMC scheme. Experiments show that MIVI not only accurately and efficiently approximates the posteriors but also facilitates designs of stochastic gradient MCMC and  Gibbs sampling transitions.
\end{abstract}
 
\end{spacing}
\section{Introduction}\label{sec:intro}\vspace{-1mm}
Markov chain Monte Carlo (MCMC) has become a reference method for Bayesian inference, especially for tasks requiring high-quality uncertainty estimation. 
However, its applications to modern machine learning problems are challenged by complex models and big data. A primary reason is that MCMC is often restricted to reversible ergodic chains, like Metropolis-Hastings (MH) \citep{metropolis1953equation,hastings1970monte} and Gibbs sampling \citep{geman1984stochastic}, which require evaluating the likelihood  over the whole data set.
A number of MCMC schemes escaping reversibility with theoretical and/or empirical supports \citep{bierkens2019zig,chen2013accelerating,neal1998suppressing} bring about considerable advantages such as accelerated mixing and enhanced adaptability to non-conjugate models, but their designs often demand significant efforts to achieve both efficacy and efficiency. 

Stochastic gradient MCMCs (SG-MCMCs) \citep{welling2011bayesian,ding2014bayesian,ma2015complete,li2016preconditioned}, which exploit the gradient information and neglect MH rejection steps, have been widely adopted for big data applications. Starting from arbitrary initial samples, 
SG-MCMCs move towards the stationary distribution via a random walk with 
step sizes annealed to zero. Thus it may either need labor-intensive tuning of the step-size annealing schedule, 
or easily suffer from slow mixing or high approximation errors.
Variational inference (VI) approximates posterior $p(\zv\given \xv)$ with variational distribution ${\textstyle q(\zv)}$ by minimizing KL${\textstyle \left(q(\zv)\,||\, p(\zv\given \xv)\right)}$, the Kullback--Leibler (KL) divergence from  $p(\zv\given \xv)$ to ${\textstyle q(\zv)}$ \citep{jordan1999introduction,blei2017variational}. Though ${\textstyle q(\zv)}$ may underestimate uncertainty if its presumed distribution family ($e.g.$, diagonal Gaussian) is not flexible enough, VI is often much faster in finding a high posterior density region than MCMC which explores the whole parameter space by random jumps based on local information~\citep{robert2018accelerating}.

Inspired by the advantages of MCMC and VI that overcome each other's limitations, we start a Markov chain with initial values drawn from an optimized variational distribution $q(\zv)$ so that the convergence can be expedited. 
If marginal distributions of this $q(\zv)$-mixed Markov chain are more flexible than the variational distribution family of $q$,     
there emerge interesting research questions: Can the framework of VI 
curb such a Markov chain from running wild as well as drive it towards the posterior? If yes, how can we design such a Markov chain that is (richly) parameterized and 
jointly optimized with $q(\zv)$
to deliver posterior approximations as good as valid MCMCs?
Therefore, we are motivated to propose MCMC-interactive variational inference (MIVI) for efficient and high-quality uncertainty estimation. 
MIVI admits stochastic-gradient optimizations with a small number of MCMC updates of $q(\zv)$ and allows fast posterior sampling without keeping track of MCMC iterations. Furthermore, leveraging MCMCs that converge to the true posterior, we provide the parameterized Markov chain with an appropriate but adequate amount of flexibility to ease its optimization. 

We encounter two-way difficulties when 
MCMC interacts with VI for mutual improvement. First, given an MCMC scheme, it is nontrivial to minimize the KL divergence from the posterior to 
the marginal distribution of the chain, because the density of the latter is often implicitly-defined by MCMC transitions. Second, even if the KL divergence is computable, it can be arduous to design a Markov chain that moderately improves $q$ without worrying about mode collapse or overdispersion. 
Our proposed MIVI has well addressed these challenges.
To avoid calculating the KL divergence, we use a discriminator 
to estimate a log density ratio \citep{mescheder2017adversarial}.
To design a Markov chain that effectively improves $q$,  MIVI borrows the idea of MCMC and $\text{(semi-)implicit}$ VIs \citep{ranganath2016hierarchical,tran2017hierarchical,yin2018semi,molchanov2018doubly,titsias2018unbiased} and strikes a balance between flexibility and convergence to the true posterior. Concretely, we replace unfavorable components of a valid MCMC scheme by (richly) parameterized functions that is to be learned in the VI framework; we learn step sizes of a SG-MCMC for general-purpose inference and design model-specific Gibbs-sampling-like Markov chains for more accurate estimations at lower computing cost. More importantly, the optimized chain in MIVI can used separately as a valid MCMC. To the best our knowledge, MIVI is the first VI algorithm to utilize Gibbs sampling transitions and to facilitate their potential inspirition-driven designs.


\vspace{-1.5mm}
\section{Method description}\label{sec:method}\vspace{-1.5mm}


MIVI is constructed by a variational distribution ${\textstyle q_{\phi}}$ mixed with a Markov chain, where ${\textstyle q_{\phi}}$ parameterized by $\phi$ is used to initialize $T\in \mathbb{Z}_+$ transitions of the chain.
We use  the marginal distribution of the chain at time $T$ as a refined variational distribution, written as ${\textstyle \tilde{q}^{(T)}_{\eta,\phi}(\zv)=\int h^{(T)}_\eta(\zv\given \zv_0) q_{\phi}(\zv_0) d\zv_0 }$ where ${\textstyle h^{(T)}_\eta}$ parameterized by ${\textstyle \eta}$ is the kernel of $T$ transitions of the chain. We show how to optimize $\phi$ and ${\textstyle \eta}$ in the framework of VI given valid formulations of ${\textstyle h^{(t)}_\eta}$, as well as how to formulate such ${\textstyle h^{(t)}_\eta}$ for monotonically non-increasing ${\textstyle\mbox{KL}(\tilde{q}^{(t)}_{\eta,\phi}(\zv)\,||\,  p(\zv \given \xv))}$ as ${\textstyle t}$ grows. With theoretical support provided, the short Markov chain admits extrapolation and fast posterior simulation. We defer all the proofs to Appendix. When there is no ambiguity, we omit the superscript $(T)$ and denote for brevity the marginal distribution by ${\textstyle \tilde{q}_{\eta,\phi}}$ and the transition by ${\textstyle h_\eta}$.

We first focus on optimizing $\phi$ and $\eta$ given a valid $h_\eta$. Suppose $p_\theta(\xv,\zv)=p_{\theta}(\xv\given \zv)p(\zv)$ is the joint likelihood of data $\xv$ given $\zv$ and prior $p(\zv)$. 
We optimize $\theta$, $\phi$, and $\eta$ to maximize the ELBO:
\begin{align}
{\textstyle
\max\limits_{\theta,\phi,\eta} \mathbb{E}_{\tilde{q}_{\eta,\phi}(\zv)}\log\frac{p_\theta(\xv, \zv)}{\tilde{q}_{\eta,\phi}(\zv)} =\max\limits_{\theta,\phi,\eta} \mathbb{E}_{\tilde{q}_{\eta,\phi}(\zv)}\log\frac{p_\theta(\xv, \zv)}{ q_{\phi}(\zv)} - \mbox{KL}(\tilde{q}_{\eta,\phi}(\zv) \,||\, q_{\phi}(\zv)).\label{eq:objective}
}
\end{align}
The first term on the right-hand side of \eqref{eq:objective} is simple to estimate if the transition $h_\eta$ is reparameterizable. Difficulty lies in $\mbox{KL}(\tilde{q}_{\eta,\phi}(\zv) \,||\, q_{\phi}(\zv))$ because marginal distribution $\tilde{q}_{\eta,\phi}(\zv)$ is not always in closed form. To circumvent the difficulty we use a discriminator to estimate  $\log\frac{\tilde{q}_{\eta,\phi}(\zv)}{q_{\phi}(\zv)}$ which only requires to draw random samples from the two distributions \citep{mescheder2017adversarial}. Specifically, with fixed $\tilde{q}_{\eta,\phi}(\zv)$ and $q_{\phi}(\zv)$, an optimal discriminator that is able to distinguish samples from the two distributions will be $D^*(\zv) = \log \tilde{q}_{\eta,\phi}(\zv) -\log q_{\phi}(\zv)$ that solves
\begin{align}
{\textstyle
\max\nolimits_D  \mathbb{E}_{\tilde{q}_{\eta,\phi}(\zv)}\log \sigma(D(\zv)) + \mathbb{E}_{q_{\phi}(\zv)}\log (1-\sigma(D(\zv))),
\label{eq:gan_loss}
}
\end{align} 
where $\sigma(\cdot)$ is the sigmoid function. Consequently, \eqref{eq:objective} turns out to be 
\begin{align} 
{\textstyle
\max\nolimits_{\theta,\phi,\eta} \mathbb{E}_{\tilde{q}_{\eta,\phi}(\zv)}\left[\log{p_{\theta}(\xv, \zv)}-\log{q_{\phi}(\zv)} - D^*(\zv) \right].\label{eq:obj_discr}
}\end{align} 
\vspace{-8mm}
\subsection{Optimization}\vspace{-2mm}
Theoretically, the ELBO \eqref{eq:obj_discr} can be maximized if the discriminator is flexible enough. In practice, however, the saturation of the sigmoid function in the cross-entropy loss of \eqref{eq:gan_loss} undermines the power of $D$ to distinguish samples from $q_{\phi}$ and $\tilde{q}_{\eta,\phi}$. Concretely, if $q_{\phi}$ is far from $\tilde{q}_{\eta,\phi}$, the optimization procedure encourages large $D$, driving $\sigma(D)$ to approach value $1$ which is a saturation region of the sigmoid function, and consequently, the diminished gradient significantly slows down $D$ from getting bigger. Meanwhile, when maximizing the ELBO of \eqref{eq:obj_discr} with an under-optimized discriminator $D$ for $\mbox{KL}(\tilde{q}_{\eta,\phi}(\zv) \,||\, q_{\phi}(\zv))$, a small increase of $D$ cannot compensate for a much larger increase of the cross entropy $-\mathbb{E}_{\tilde{q}_{\eta,\phi}}\log{q_{\phi}(\zv)}$ if $q_{\phi}$ and $\tilde{q}_{\eta,\phi}$ are too far from each other. In short, a big discrepancy between $q_{\phi}$ and $\tilde{q}_{\eta,\phi}$ impedes optimizing the discriminator and a poor discriminator further spaces the two distributions. This vicious circle often makes \eqref{eq:obj_discr} fail to increase $\mathbb{E}_{\tilde{q}_{\eta,\phi}(\zv)}\log{p_\theta(\xv, \zv)}$ and hence brings about poor estimations of $\tilde q_{\eta, \phi}$ and $q_{\phi}$ that drift apart from each other.

Even if the discriminator is so flexible that it is unaffected by the vicious circle, optimizing \eqref{eq:obj_discr} by gradient ascent with respect to $\phi$ can be intractable because $D^*$ itself, found by \eqref{eq:gan_loss}, depends on $\phi$. The problem of calculating this gradient 
cannot be solved by the strategy of \citet{mescheder2017adversarial} after the Markov chain is introduced. To circumvent the two aforementioned difficulties when using the discriminator, MIVI reformulates the objective by maximizing 
a lower bound of \eqref{eq:obj_discr} with respect to $\theta$ and $\eta$ given optimal $D^*$ and $\phi^*$ that are obtained by  two  auxiliary optimization problems. This lower bound and the two auxiliary optimization problems are expressed as 
\begin{align}
&\max\nolimits_{\theta,\eta} \mathbb{E}_{\tilde{q}_{\eta,\phi^*}(\zv)}\left[\log{p_{\theta}(\xv, \zv)}-\log{q_{\phi^*}(\zv)} - D^*(\zv) \right] \label{eq:obj_lrbnd}, \\
&D^* = \arg\max\nolimits_D  \mathbb{E}_{\tilde{q}_{\eta,\phi^*}(\zv)}\log \sigma(D(\zv)) + \mathbb{E}_{q_{\phi^*}(\zv)}\log (1-\sigma(D(\zv)))\label{eq:gan_loss2}, \\
&\phi^* = \arg\min\nolimits_{\phi} -\mathbb{E}_{\tilde{q}_{\eta,\phi}(\zv)}\log{q_{\phi}(\zv)}. \label{eq:xentropy}
\end{align}
It is straightforward to take the gradient of \eqref{eq:obj_lrbnd} and \eqref{eq:xentropy} with respect to  $\theta$ and $\phi$, respectively.
More importantly, the following property overcomes the difficulty in taking the gradient of $D^*$ with respect to $\eta$ when maximizing \eqref{eq:obj_lrbnd} under the assumption of reparameterizable Markov chain transitions.
\begin{property}\label{prop:gradient} \vspace{-1mm}
Suppose 
$h_\eta$
is reparameterizable, which means there exists a deterministic vector-valued function ${\textstyle f_\eta}$ and a random vector $\varepsilon$ such that
${\textstyle \zv^{(T)}\sim \tilde q_{\eta,\phi}}$ is equivalent to 
${\textstyle \zv^{(T)}=f_\eta(\zv^{(0)}, \varepsilon)}$
where ${\textstyle \zv^{(0)}\sim q_{\phi}(\zv)}$. The gradient of \eqref{eq:obj_lrbnd} with respect to ${\textstyle\eta}$ is equal to $$
\resizebox{0.91\hsize}{!}{$
\textstyle \small
\mathbb{E}_{\varepsilon}  \big[ \nabla_{\eta}\log p_{\theta}(\xv, f_\eta(\zv^{(0)}, \varepsilon)) -\nabla_{\eta}\log q_{\phi^*}(f_\eta(\zv^{(0)},  \varepsilon))
-(\nabla_{\eta}f_\eta(\zv^{(0)}, \varepsilon)) ( \frac{dD^*(\zv)}{d \zv}\,\big |\,_{\zv=f_\eta(\zv^{(0)},  \varepsilon)})\big].$}
$$
\end{property}\vspace{-3mm}

\subsection{Formulation of Markov chain transitions}\vspace{-2mm}
We have discussed the optimization of $\theta$, $\phi$ and $\eta$ in MIVI. But MIVI makes sense only if $\tilde q_{\eta,\phi}$ is a better posterior approximation than $q_\phi$. Yet to be determined is the formulation of a valid transition $h_\eta$ that keeps pushing $\tilde q_{\eta,\phi}(\zv)$ closer to $p(\zv\given \xv)$ and thus enables extrapolation of the short Markov chain. We utilize stochastic gradient Langevin dynamics (SGLD)  \citep{welling2011bayesian} as a general-purpose solution and Gibbs sampling
for model-specific but more efficient inference. 

{\bf 
SGLD
~~~} 
So far $h_\eta$ being reparameterizable is the only assumption of MIVI on the Markov chain.
Consequently, SGLD can be incorporated in MIVI and universally applied, as it approximates posteriors with stochastic gradient descent and injected Gaussian noise. Concretely, for a mini batch $x$ of size $n$ from training data of size $N$, a variable $z$ at discrete time $t$ of SGLD is updated by
\begin{align}
{\textstyle
z^{(t)} = z^{(t-1)}+
 \frac{\eta_t}{2}[\nabla_{z} \log p(z^{(t-1)}) +\frac{N}{n} \nabla_{z} \log p(x \given z^{(t-1)})] + \epsilon_t,~~~\epsilon_t \sim \mathcal{N}(0,\eta_t)\label{eq:sgld}
 }
\end{align}
where $\eta_t$ is the step size at time $t$.
With a long run and diminishing $\eta_t$, SGLD proceeds through two phases \citep{welling2011bayesian}: the first is the phase of stochastic optimization in which $p(\xv,\zv)$ is being maximized, and the second is the phase of Langevin dynamics 
in which a random walk is approximating the posterior sampling.   With standard assumptions \citep{khasminskii2011stochastic,vollmer2016exploration} to guarantee ergodicity and diminishing step sizes, 
SGLD converges to a stationary distribution that well approximates the true posterior. Particularly, \citet{teh2016consistency} provide conditions
under which SGLD converges to the posterior and find consistent posterior estimators with asymptotic normality. 
The following lemma validates the use of any $T$ steps of SGLD transitions in MIVI and an extrapolation ($i.e.,$ more than $T$ steps of transitions). 
\begin{lem}[Page 81 of \citet{cover2006elements}]\label{lemma:SG-MCMC}\vspace{-1mm}
Suppose $\zv$ are variables on a Markov chain $M$ with the stationary distribution $\pi(\zv)$.
Let $\mu^{(t)}$ be any distribution on the state space of $M$ at $t$ and $\mu^{(t+1)}$ be the marginal distribution after one transition from $\mu^{(t)}$. Let $q$ denote the mass/density function of variables $\zv^{(t)}\sim \mu^{(t)}$ or $\zv^{(t+1)}\sim \mu^{(t+1)}$. We have ${\textstyle  \emph{\mbox{KL}}(q(\zv^{(t)})\,||\, \pi(\zv))\geq\emph{\mbox{KL}}(q(\zv^{(t+1)})\,||\, \pi(\zv)).}$
\end{lem}\vspace{-1mm}
Specifically, Lemma \ref{lemma:SG-MCMC} sheds light on SGLD's continuously refined $\tilde q_{\eta,\phi}^{(t)}$ in terms of its KL divergence from the stationary distribution, for not only $t\leq T$ in MIVI, but also $t>T$ in extrapolation as long as the step sizes appropriately anneal. The guaranteed superiority of $\tilde{q}_{\eta,\phi}^{(t)}$ over $q_{\phi}$, however, is not enough; running a finite number of transitions, people also seek a balance of fast convergence (low variance) and small discretization errors (low bias) by a good selection of step sizes which may need to be tuned labor-intensively. To this end, 
MIVI incorporates $T$ transitions of SGLD, sets $\eta$ of $h_{\eta}$ as step sizes $\{\eta_1,\ldots,\eta_T\}$, and optimizes $\eta$ by \eqref{eq:obj_lrbnd} for a good bias-variance tradeoff. Moreover, with $q_{\phi}$ of MIVI locating high posterior density regions, the $T$ transitions start from the second phase of SGLD and the optimized step sizes can be leveraged to extrapolate the markov chain to any length $t$, $T<t<\infty$ for even better posterior estimations. 

{\bf Gibbs sampling~~~}
In addition to SGLD, Lemmas \ref{thm:gibbs} and \ref{cor:collapsed} show MIVI can utilize reparameterizable Gibbs sampling transitions to keep improving  ${\textstyle \tilde{q}^{(t)}_{\eta,\phi}(\zv)}$ as ${\textstyle t}$ increases and, more significantly, facilitate an efficient design of MCMC alternative to Gibbs sampling. Specifically, it is implied that for a Markov chain of variables ${\textstyle \zv}$ of interest and (auxiliary) variables ${\textstyle \wv}$, using a valid Gibbs sampling transition for ${\textstyle \zv}$ will keep pushing its marginal distribution closer to the posterior as long as the Markov chain's transition for ${\textstyle \wv}$ is good enough.
\begin{lem}\label{thm:gibbs}\vspace{-1mm}
Suppose the transition of a Markov chain $M$ of $(\wv, \zv)$ at  time $t+1$ is $r$ such that $r(\wv^ {(t+1)}, \zv^{(t+1)}\given \wv^{(t)},\zv^{(t)}) = r(\zv^{(t+1)} \given \wv^{(t)})r(\wv^{(t+1)}\given \zv^{(t+1)})$. Let $\mu^{(t)}$ be any distribution on the state space of $M$ at $t$ and $\mu^{(t+1)}$ be the marginal distribution after one transition from $\mu^{(t)}$. Let $q$ denote the joint mass/density function and thus $\textstyle q(\wv^{(t)},\zv^{(t)}, \wv^{(t+1)},\zv^{(t+1)}) = q(\wv^{(t)},\zv^{(t)})r( \wv^{(t+1)},\zv^{(t+1)}\given \wv^{(t)},\zv^{(t)}).$ If $r(\zv\given \wv)$ is the conditional distribution of $\zv$ given $\wv$ and hence a valid transition of $\zv$ in a Gibbs sampler $G$ that converges to the posterior $\pi(\wv,\zv)$, then
${\textstyle 
\mbox{KL}(q(\wv^{(t)},\zv^{(t)})\,||\, \pi(\wv,\zv))\geq \mbox{KL}(q(\wv^{(t)},\zv^{(t+1)})\,||\, \pi(\wv,\zv)).    
}
$
\end{lem}\vspace{-1mm}
\begin{lem}\label{cor:collapsed}\vspace{-1mm}
With all the assumptions in Lemma \ref{thm:gibbs}, suppose $\mu'^{(t)}$ is the posterior with density $\pi(\wv,\zv)$ and is at time $t$ of $M$. $q'$ denotes the joint mass/density of variables from $\mu'^{(t)}$ and $\mu'^{(t+1)}$. 
If $r(\wv \given \zv)$ at time $t$ is close to $\pi(\wv|\zv)$ in the sense that  $\mathbb{E}_{q(\zv^{(t)})} \left[ \mbox{KL}(r(\wv\given \zv^{(t)}) \,||\, \pi(\wv\given \zv^{(t)})) \right] \leq \mathbb{E}_{q(\zv^{(t+1)})}\left[\mbox{KL}(q(\wv^{(t)}\given \zv^{(t+1)}) \,||\, q'(\wv^{(t)}\given \zv^{(t+1)}))\right]$, then
$ 
{\textstyle 
 \mbox{KL}(q( \zv^{(t)})\,||\, \pi(\zv)) \geq \mbox{KL}(q(\zv^{(t+1)})\,||\, \pi(\zv)).
} 
$ 
\end{lem}\vspace{-1mm}
As shown by Lemma \ref{thm:gibbs} and \ref{cor:collapsed}, if $(\wv,\zv)$ are variables of interest and the full conditional distribution $p(\zv\given \wv,\xv)$ is reparameterizable, we can use $p(\zv\given \wv,\xv)$ as the transition of $\zv$ in the Markov chain. Furthermore, a (richly) parameterized transition function $h^{(1)}_{\eta}(\wv \given \zv,\xv)$ is learned by MIVI to well approximate the full conditional distribution $p(\wv\given \zv, \xv)$ so that, by Lemma \ref{cor:collapsed}, $\zv^{(t)}$ approaches to the true posterior $p(\zv\given \xv)$ as $t$ increases, not only within the $T$ transitions of MIVI, but also for $t>T$ when the extrapolated chain serves as an MCMC scheme. This is especially useful when we care about posterior estimates of $\zv$ more than $\wv$. Moreover, iterating $p(\wv \given \zv, \xv)$ and $h^{(1)}_\eta(\wv\given \zv,\xv)$ is an efficient MCMC scheme with fast mixing because the optimized  $h^{(1)}_\eta(\wv\given \zv,\xv)$ by MIVI has located high density regions of $p(\wv\given \zv, \xv)$.
We provide in Section \ref{sec:experiment} specific applications where $\wv$ are auxiliary variables enabling a closed-form reparameterizable full conditional distribution of $\zv$. 

\subsection{MIVI implementation}\vspace{-2mm}
Instead of keeping the discriminator $D$ and ${\phi}$ optimal in every epoch when optimizing $\theta$ and $\eta$, we regard the problem as a three-player game analogous to the two-player game of \citet{mescheder2017adversarial} in order to reduce the computing cost:
1) Given $D$ and $\phi$, we optimize $\eta$ and $\theta$ to maximize ELBO  \eqref{eq:obj_lrbnd}. 
2) Given $\eta$,  we optimize $\phi$  to reduce the discrepancy between $q_{\phi}$ and $\tilde{q}_{\eta,\phi}$ measured by the cross entropy \eqref{eq:xentropy}. 
3) The discriminator $D$ tries to differentiate samples from $\tilde{q}_{\eta,\phi}$ and $q_{\phi}$. Note that $\eta$ and $\phi$ are learned adversarially and the game terminates at a saddle point that is a maximum of \eqref{eq:obj_lrbnd} with respect to $\eta$'s strategy and a minimum of \eqref{eq:xentropy} with respect to $\phi$'s strategy. The ELBO of MIVI, $\mathbb{E}_{\tilde{q}_{\eta,\phi}(\zv)}\log\frac{p_{\theta}(\xv,\zv)}{\tilde q_{\eta,\phi}(\zv)} $, is bounded above as in Property \ref{prop:bound}. 
\begin{property}\label{prop:bound}  \vspace{0mm}
${\textstyle  \mathbb{E}_{\tilde{q}_{\eta,\phi}(\zv)}\log\frac{p_{\theta}(\xv,\zv)}{\tilde q_{\eta,\phi}(\zv)}
\leq \mathbb{E}_{\tilde{q}_{\eta,\phi}(\zv)}\log\frac{p_{\theta}(\xv,\zv)}{q_{\phi}(\zv)}}.$ 
\end{property}\vspace{-2mm}

The upper bound 
together with saturation of $\sigma(D)$ provides a fast pre-training strategy. Concretely, given $\phi^*$ and $D^*$ we assume $\sigma(D^*)$ saturates such that $\mathbb{E}_{\tilde{q}_{\eta,\phi^*}(\zv)} D^*(\zv) \leq c$  for some positive constant $c$ (that may depend on $\eta$ and $\phi^*$). Consequently, \eqref{eq:obj_lrbnd} is bounded between ${\textstyle ( \mathbb{E}_{\tilde{q}_{\eta,\phi^*}(\zv)}\log\frac{p_{\theta}(\xv,\zv)}{q_{\phi^*}(\zv)}-c)}$ and ${\textstyle \mathbb{E}_{\tilde{q}_{\eta,\phi^*}(\zv)}\log\frac{p_{\theta}(\xv,\zv)}{q_{\phi^*}(\zv)}}$ which, instead, can be optimized to avoid potentially the most time-consuming training of $D$. We summarize the implementation of MIVI as Algorithm \ref{algorithm} in Appendix. We find that MIVI is numerically stable and converges 
fast as shown in Section \ref{sec:experiment}.

\section{Related work and contribution}\label{sec:relatedwork}\vspace{-1.5mm}
Using a discriminator to approximate a hard-to-compute KL divergence was first introduced by \citet{mescheder2017adversarial} that enable an arbitrarily flexible variational distribution. It is also adopted by \citet{li2017approximate} where the variational posterior is supervised by SG-MCMC. But in their training procedure
the discriminator and variational parameters are entangled in a way that makes it 
difficult to rigorously calculate the gradient of the objective function. By contrast, 
we reformulate the objective with auxiliary optimization problems and provide rigorously derived gradients. 
Learning step sizes of SGLD by VI has been explored by \citet{gallego2019variationally} and \citet{nijkamp2020learning}. The former utilizes the Gaussianity of SGLD transitions, and the latter regard SGLD as a normalizing flow that assumes a volume-preserving invertible transformation. Both methods depend on the good properties of SGLD. Comparatively, the reformulated optimizations of MIVI make it well adapted to different kinds of MCMCs with reparameterizable transitions, so that many SG-MCMCs, like Hamiltonian and Langevin dynamics, and Gibbs sampling schemes can be incorporated.

While VI and MCMC have complementary properties, existing works combining the two have primarily studied one-way improvement. As for utilizing MCMC to facilitate VI,  
a common practice is using the refined MCMC marginal distribution to guide and improve the variational distribution. \citet{ruiz2019contrastive} minimize the discrepancy between the variational and a marginal distribution of Hamilton Monte Carlo (HMC) using the contrastive divergence without explicitly computing the KL divergence.  \citet{titsias2017learning} implicitly augments the variational distribution by MCMC and a model-based reparameterization. \citet{salimans2015markov} incorporate in VI finite steps of MCMC and the MCMC samples are inferred as auxiliary variables; HMC is adopted to illustrate this idea and is related to normalizing flow. Generally, \citet{pmlr-v37-rezende15} write Hamiltonian and Langevin dynamics as infinitesimal flows; both flows can be used in VI  for a tighter ELBO and the inference requires volume-preserving invertible transformations. \citet{zhang2018ergodic} construct measure preserving flows and utilize distribution preservation 
of Hamilton Monte Carlo. \citet{chen2017continuous} propose the use of Langevin dynamics as a way to transit from one latent variable to the next to improve 
variational autoencoders (VAEs). 

On the other hand, research of using VI to facilitate MCMC 
includes \citet{variationalmcmc2001} that use a variational distribution as the MH proposal to alleviate the poor scaling with dimension of the independent Metropolis algorithm.   
\citet{habib2018auxiliary} learn a lower-dimensional embedding of the parameters of interest by VI 
to accelerate MCMC mixing. Several works share the idea  of providing MCMC proposals with more flexibility by introducing auxiliary variables \citep{maddison2017filtering,naesseth2018variational,anh2018autoencoding}. 
In comparison, we fulfill mutual improvement of VI and MCMC by MIVI. Being a marginal distribution of valid MCMCs, the variational distribution of MIVI gets closer to the posterior. MIVI replaces unfavorable parts (like unknown, non-reparameterizable or manually tuned transitions, see Section \ref{sec:experiment}) of MCMCs by (richly) parameterized functions and learns them in the framework of VI. In this way, MIVI facilitates designs of MCMCs. More importantly, with theoretical support, the chain in VI can be extrapolated and used as an efficient alternative to well established MCMCs. In addition, to the best of our knowledge, MIVI is the first method to combine VI and Gibbs sampling.

\section{Experiments}\label{sec:experiment}\vspace{-1.5mm}
We first use toy data (deferred to Appendix) 
and a negative binomial (NB) model to illustrate the flexibility  of MIVI incorprating a few SGLD transitions. 
Next, we use both Bayesian logistic and bridge regression to show MIVI and Gibbs sampling facilitate each other when some of the Gibbs sampling transitions are unknown or not reparameterizable. In addition, we provide experiments of variational autoencoders (VAEs)  \citep{kingma2013auto} by MIVI and demonstrate its remarkable performance compared to existing state-of-the-art algorithms. We use \textit{Adam} \citep{kingma2014adam} to otpimize $\theta$, $\phi$, $\eta$ and $D$ with the learning rate as $0.001$. Throughout this section unless specified, the prior $p(\zv)$ used in Gibbs sampling, mean-field VI (MFVI), and MIVI is $\mathcal{N}(0,I)$ for real-valued $\zv$ and the variational distribution $q_{\phi}(\zv)$ is a diagonal Gaussian whose mean and log of variances constitute $\phi$. We set $\eta$ as the step sizes of SGLD if incorporated in MIVI, and for simplicity, a time-invariant step size is set and learned. The learned step size can initialize an appropriate decay, like the one suggested by \citet{teh2016consistency}.  Also see \citet{vollmer2016exploration} for theoretical analysis of SGLD with a fixed step size.
More experiment settings are deferred to Appendix.

\subsection{Negative binomial model}\label{sec:NB}\vspace{-1mm}
We draw 1,000 random samples from negative binomial (NB) distribution $\mbox{NB}(x \given r=2, p=0.7)$ whose probability $p(x)=\frac{\Gamma(x+r)}{\Gamma(r) x!}p^x(1-p)^r$ for $x\in {0}\cup\mathbb{Z_+}$.  We use $r\sim \mbox{Gamma}(0.1,0.1)$ and $p\sim\mbox{Beta}(0.1,0.1)$ as the prior. Posteriors of $r$ and $p$ under the NB model are estimated by  Gibbs sampling \citep{zhou2012lognormal}, MFVI, and MIVI. We set $z=(\log(r), \mbox{logit}(p))$ in MFVI and MIVI that incorporates $T=10$ SGLD transitions.
Shown in Figure \ref{fig:NBLogit} (a) are the estimated density contour plots of $(r,p)$ by Gibbs sampling and the one transformed from  $(\log(r),\mbox{logit}(p))\sim q_{\phi}$ by MFVI. Analogously plotted in Figure \ref{fig:NBLogit} (b) are the densities of $(r,p)$ resulting from $q_{\phi}$ and $\tilde{q}_{\eta,\phi}$ by MIVI. The negative correlation in the posterior of $r$ and $p$ as shown by Gibbs sampling has been well recovered by $\tilde{q}_{\eta,\phi}$ in MIVI. Furthermore, the diagonal Gaussian $q_{\phi}$ in MFVI has underestimated the parameter uncertainty whereas $q_{\phi}$ of the same family in MIVI gives much better variance estimations because 
MIVI restrains the discrepancy between $q_{\phi}$ and $\tilde{q}_{\eta,\phi}$.

\begin{figure}[!t]\vspace{-3mm}
 \centering
 \begin{subfigure}[t]{0.245\textwidth}
 \centering
\includegraphics[width=1\linewidth]{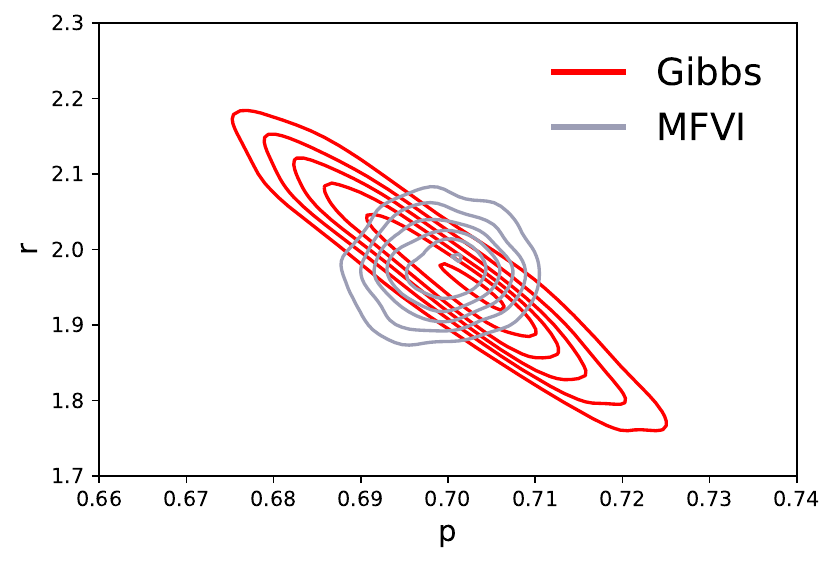}\vspace{-2.5mm} 
 \caption{\small NB: Gibbs and MFVI.}\vspace{-1mm}
 \end{subfigure}%
 \begin{subfigure}[t]{0.245\textwidth}
 \centering
\includegraphics[width=1\linewidth]{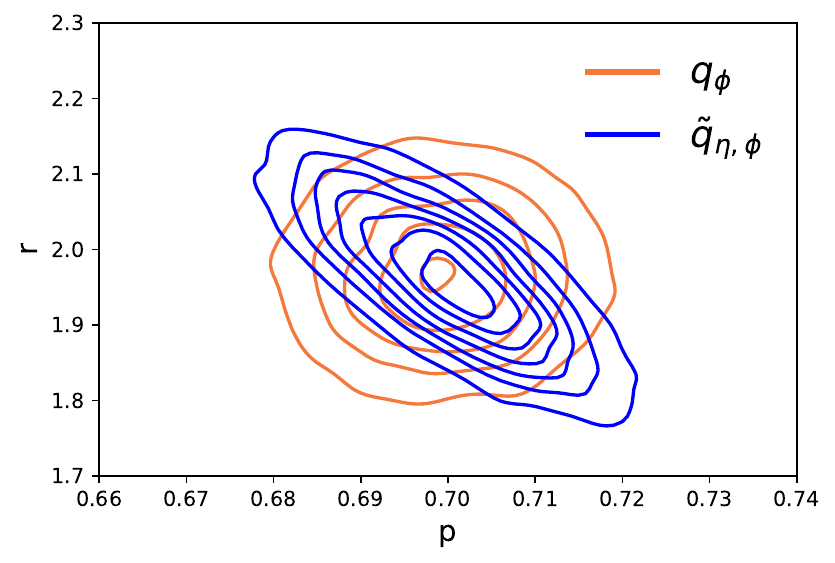}\vspace{-2.5mm} 
 \caption{\small NB: MIVI. 
 		 }\vspace{-1mm}
 \end{subfigure}%
\begin{subfigure}[t]{0.245\textwidth}
 \centering
\includegraphics[width=1\linewidth]{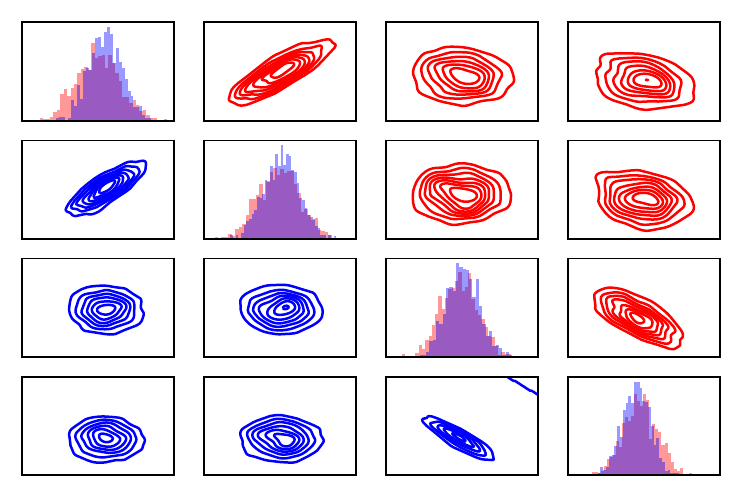}\vspace{-2.5mm} 
 \caption{\small Logistic: $\beta$.}\vspace{-1mm}
 \end{subfigure}%
 \begin{subfigure}[t]{0.245\textwidth}
 \centering
\includegraphics[width=1\linewidth]{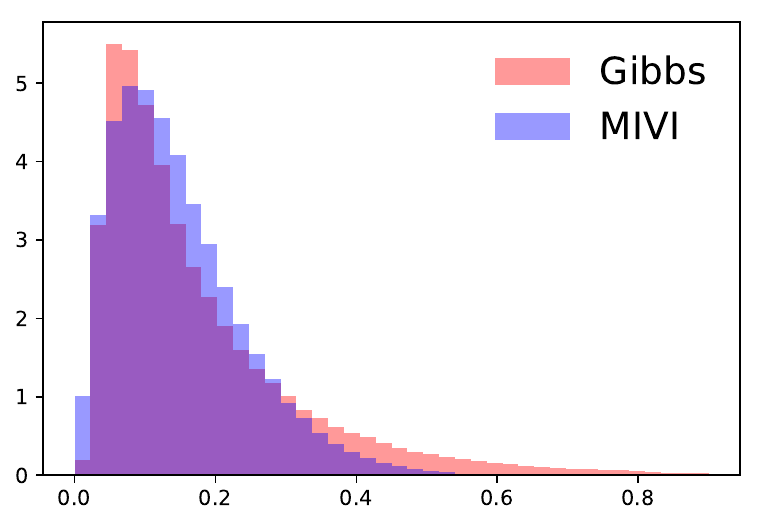}\vspace{-2.5mm} 
 \caption{\small Logistic: $\omega$.}\vspace{-1mm}
 \end{subfigure}
\caption{\small Estimated posterior densities. (a) and (b) are the estimated posteriors of $r$ and $p$ for the negative binomial model by Gibbs sampling (red), MFVI (gray) and MIVI (orange for $q_{\phi}$ and blue for $\tilde{q}_{\eta,\phi}$), respectively. (c) and (d) are the estimated posteriors of the logistic regression coefficient $\beta$ and the auxiliary Polya gamma random variable $\omega$ by Gibbs sampling (red) and $\tilde{q}_{\eta,\phi}$ of MIVI (blue).
}\label{fig:NBLogit}\vspace{-6mm}
\end{figure}\vspace{-2pt}

Next, we accentuate MIVI that uses Gibbs sampling transitions for variables $\zv$ and replaces unknown or non-reparameterizable Gibbs transitions for variables $\wv$ by (richly)  parameterized functions. Compared to SGLD, MIVI needs fewer Gibbs-sampling-like transitions without sacrificing capacity, and the inferred $\tilde q_{\eta,\phi}$ gives comparable posterior estimates to Gibbs sampling.
Examples include Bayesian logistic and bridge regression  using auxiliary variables that are difficult to find or sample.
\vspace{-2mm}
\subsection{Bayesian logistic regression}\label{sec:logistic}\vspace{-1mm}
One of the most well-known data augmentation schemes is the Polya gamma (PG) augmentation for logistic regression \citep{polson2013bayesian}, making the regression coefficients have Gaussian conditional distributions. 
Specifically, given a unique $\xv_i$, $i=1,\cdots,n$, and $y_i \sim \mbox{Bernoulli}( \sigma (\xv_i' {\beta}))$, 
$p(y_i\given x_i,\beta) = \frac{e^{y_i \xv_i'\beta}}{1+e^{\xv_i'\beta}}=\frac{e^{(y_i-\frac{1}{2})\xv_i'\beta}}{2}\int_{0}^\infty e^{- \omega_i(\xv_i'\beta)^2/2}p(\omega_i)d\omega_i$
where $p(\omega_i)$ is the density of $\mbox{PG}(1,0)$ prior on $\omega_i$. The conditional posterior of $\omega_i$ is $\mbox{PG}(1,\xv_i'\beta)$ and that of $\beta$ is a Gaussian distribution. Iterating the samplings from both distributions defines a valid Gibbs sampler 
(see Appendix for  details). However, PG distributions are not reparameterizable. Therefore, in the Markov chain of MIVI we use the Gaussian full conditional distribution as the transition for $\beta$ and a neural network $g_{\eta}$ parameterized by $\eta$ for local variables $\omega_i$'s. Specifically, concatenating $\xv_i'\beta$ and an independent Gaussian random vector $\epsilon_i$ as the input of $g_\eta$, the Markov chain in MIVI proceeds by iterating $\omega_i = g_\eta(\xv_i'\beta, \epsilon_i)$ and $(\beta \given -)\sim \mathcal{N}\left(\Sigma_{\beta}(X'\kappa + I), \Sigma_{ \beta}  \right)$,
where $\Sigma_{ \beta}  = (X'\Omega  X  +I)^{-1}$, $\Omega=\mbox{diag}(\omega_1,\ldots,\omega_n)$, and $\kappa = (y_1-0.5,\cdots,y_n-0.5)$. 

We synthesize a data set of 1,000 four-dimenstional, correlated $\xv_i\sim \mathcal{N}(0,\Sigma)$ where the elements of $\Sigma$ are $\sigma_{v,v} = 1$, $v=1,2,3,4$, $\sigma_{1,2}=\sigma_{2,1}=-0.8$, $\sigma_{3,4}=\sigma_{4,3}=0.9$ and other $\sigma_{v,v'}=0$. True $\beta = (\beta_1, \beta_2,\beta_3,\beta_4)$ is set to be $(-2,-1,1,2)$ and $y_i\sim \mbox{Bernoulli}( \sigma(\xv_i'\beta))$. Good estimations of $\beta_{1}$ and $\beta_2$ should be positively correlated and $\beta_{3}$ and $\beta_4$ negatively correlated. We run only one transition of the Markov chain in MIVI (i.e., $T=1$) and compare $\tilde{q}_{\eta,\phi}$ with Gibbs sampling. Shown in Figure~\ref{fig:NBLogit} (c) and (d), respectively, are the estimated posterior of $\beta$ and of $\omega$ averaged over data which is $\int p(\omega\given \xv)p(\xv)d\xv$. As a result, $\tilde{q}_{\eta,\phi}$ of MIVI gives rise to comparable posterior estimations to Gibbs sampling. We plot in Appendix the estimated posterior $\omega_i$ for some randomly selected $i$ which are also  similar to those from Gibbs sampling. Additionally, logistic regression of binary MNIST (3 v.s. 5) by MIVI achieves a testing accuracy of $95.79\%$ that matches $95.64\%$ from the MLE of a well-tuned $L_2$-penalized logistic regression. Also provided in Appendix are estimated distributions of $\omega_i$ from $\tilde{q}_{\eta,\phi}$ associated to randomly selected MNIST training images. Therefore, having well approximated the non-reparameterizable Gibbs sampling transition of $\omega_i$'s by a neural network, MIVI delivers posterior estimations on par with Gibbs sampling and preserves the classification capacity.

\subsection{Bayesian bridge regression}\label{sec:bridge}
\begin{figure}[!t]\vspace{-3mm}
 \centering
  \begin{subfigure}[t]{0.28\textwidth}
 \centering
\includegraphics[width=1\linewidth]{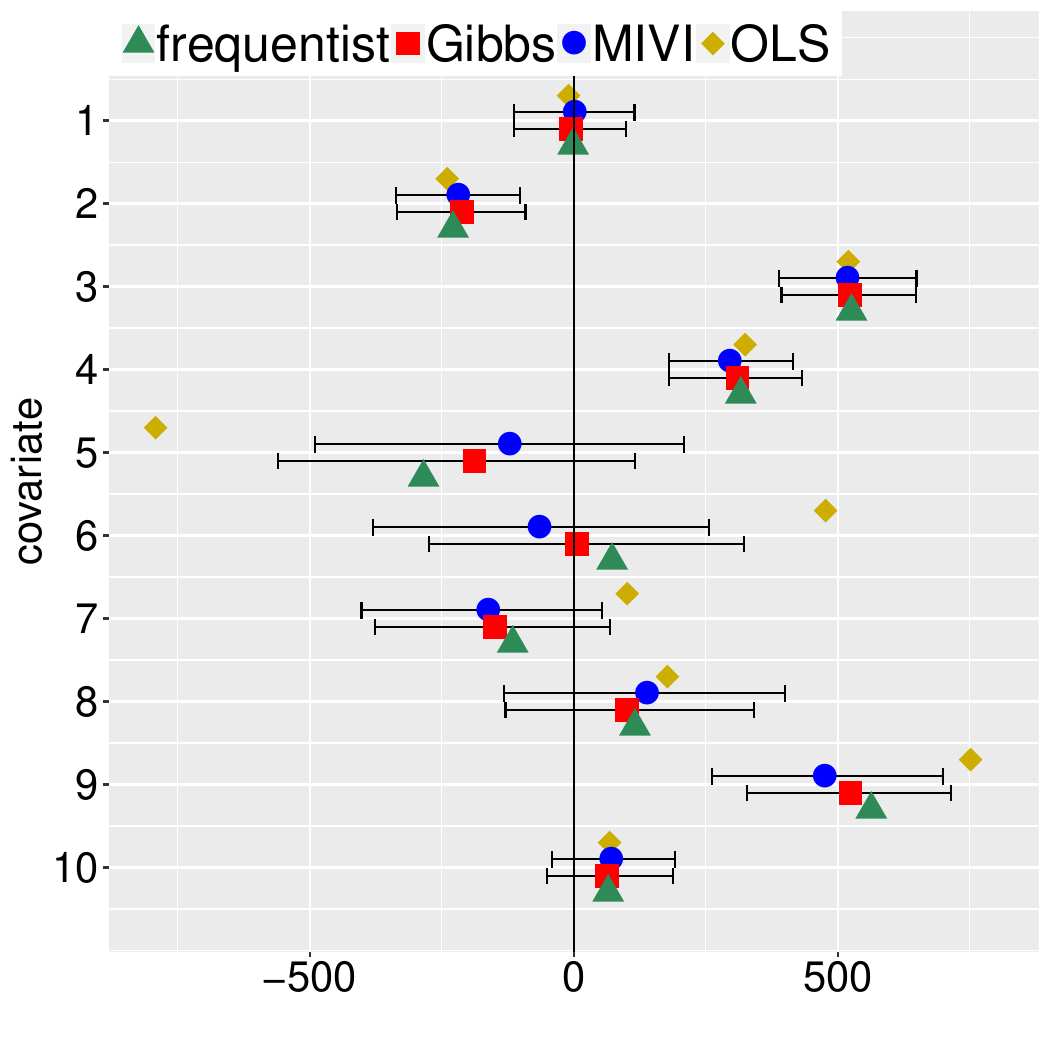}\vspace{-3mm} 
 \caption{\small $\alpha=1$ (Lasso).}\vspace{-0.5mm}
 \end{subfigure}%
 \begin{subfigure}[t]{0.28\textwidth}
 \centering
\includegraphics[width=1\linewidth]{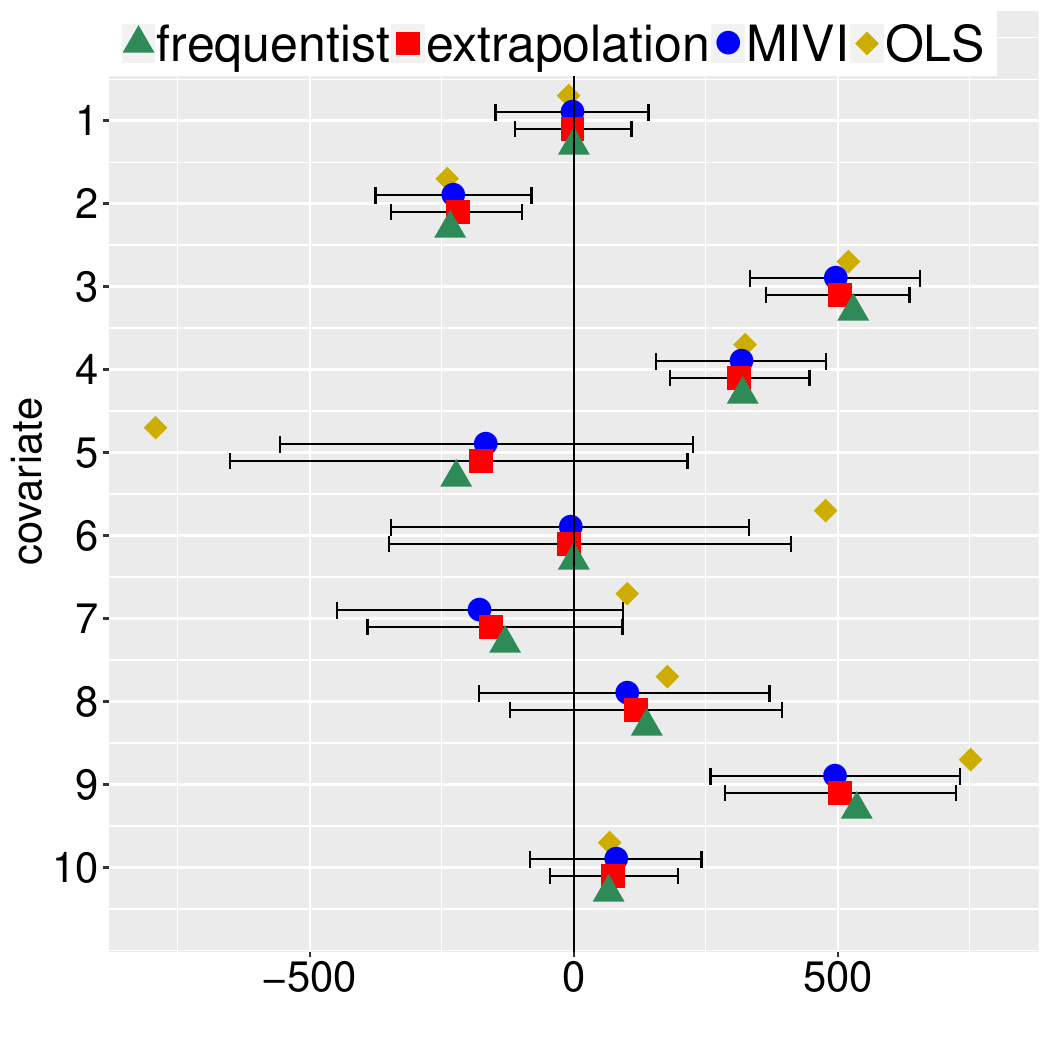}\vspace{-3mm} 
 \caption{\small $\alpha=0.5$.}\vspace{-0.5mm}
 \end{subfigure}%
 \begin{subfigure}[t]{0.28\textwidth}
 \centering
\includegraphics[width=1\linewidth]{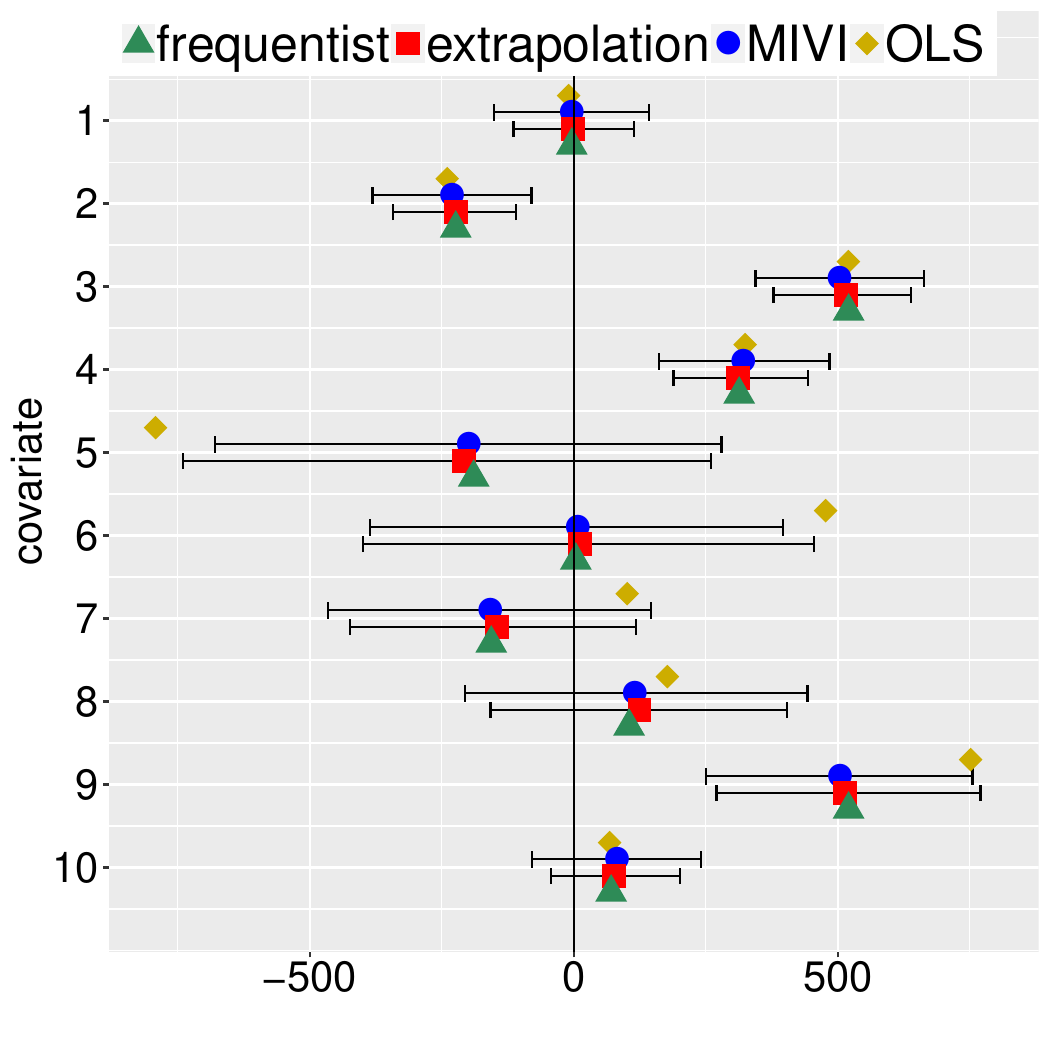}\vspace{-3mm} 
 \caption{\small $\alpha=1.5$.}\vspace{-0.5mm}
 \end{subfigure}%
 \vspace{-2mm}
\caption{\small Bridge regression of \texttt{diabetes} data. (a) is results of $\alpha=1$ (Lasso), (b) $\alpha=0.5$ and (c) $\alpha=1.5$, including point estimates of $\beta$ by a frequentist approach (green triangle) minimizing the loss function, Gibbs sampling or the extrapolated Markov chain (red square), $\tilde q_{\eta,\phi}$ of MIVI (blue dot) and OLS (yellow diamond) and the $95\%$ CIs by Gibbs (or the extrapolated chain) and MIVI.
}\label{fig:bridge}\vspace{-3mm}
\end{figure}

Next we show MIVI not only well approximates posteriors but also helps to design valid Gibbs-sampling-like MCMC when some Gibbs sampler transitions are unknown in analytic expressions.
Bridge regression tries to find $\hat\beta=(\hat\beta_1,\ldots,\hat\beta_p)$ that minimizes
$\frac{1}{2}||y-X\beta||^2 + \psi \sum_{v=1}^p |\beta_v|^a$
given the choice of $\alpha\in(0,2)$ and $\psi>0$. From a Bayesian perspective, a hirarchical model for bridge regression is $p(y \given X, \beta,\sigma) = \mathcal{N}(y\given X\beta, \sigma^2 I)$, $ p(1/\sigma^2) = \mbox{Gamma}(1/\sigma^2\given r,c)$, and $p(\beta_v \given \alpha,\rho,\sigma) \propto e^{-\rho |\beta_v/\sigma|^\alpha}$ for $v=1,\ldots,p$, 
where $r$ and $c$ are the gamma shape and rate parameter, respectively, and $\rho$ is a hyper-parameter regularizing the $L_\alpha$ norm of $\beta$. A data augmentation that writes $p(\beta_v \given \alpha,\rho,\sigma)$ as a scale mixture of normals enables conjugacy. Specifically, ${\textstyle e^{-\rho |\beta_v/\sigma|^\alpha} = \int_0^\infty e^{-\frac{\lambda_v \beta_v^2}{2\sigma^2}\rho^{2/\alpha}} g(\lambda_v)d\lambda_v \label{eq:bridge} }$, 
where $g(\lambda_v)$ is proportional to the density of a positive stable distribution with index of stability ${\alpha}/{2}$ \citep{west1987scale,polson2014bayesian}. While both the prior and  full conditional of $\beta$ are Gaussian, neither the posterior nor the full conditional distribution of $\lambda_v$ is known in closed form, which impedes an efficient Gibbs sampler under this data augmentation. 

To circumvent the unknown conditional distribution of global variables $\lambda_v$'s, we use a flexible reparameterizable distribution to approximate their marginal distribution which serves as a time-invariant transition of $\lambda_v$'s in the Markov chain of MIVI. For simplicity, we adopt Weibull distributions as $\lambda_v\sim\mbox{Weibull}(a_v,b_v)$, which is equivalent to $\lambda_v = a_ve^{\log(-\log u)/b_v}$, $u\sim \mbox{Uniform}(0,1)$, but other flexible distributions on $\mathbb{R}_+$, like a neural network with random noise as input, also work as long as they are reparameterizable. Given $\lambda_v$'s, $\beta$ and $\sigma^2$ are updated according to their full conditional distributions. Concretely, the Markov chain of MIVI proceeds by iterating
\begin{align}
&{\textstyle (\lambda_v\given -) \sim \mbox{Weibull}(a_v, b_v),~ v=1,\ldots,p,~~~~~~
(\beta\given -)   \sim \mathcal{N} (\Sigma X'y, \sigma^2\Sigma ),\nonumber}\\
&(1/\sigma^2\given-) \sim \textstyle \mbox{Gamma} (r+\frac{n+p}{2}, c+\frac{1}{2}||y-X\beta||^2+\frac{1}{2}\sum_{v=1}^p \rho^{2/\alpha}\lambda_v\beta_v^2 ),\label{eq:bridge_gibbs}
\end{align}
where $\Sigma=(X'X+\rho^{2/\alpha}\Lambda)^{-1}$, $\Lambda=\mbox{diag}(\lambda_1,\cdots,\lambda_p)$, and $n$ is the number of observations. 
With $\eta=(\log a_1,\log b_1,\ldots,\log a_p,\log b_p)$ optimized by MIVI, the Markov chain can be extrapolated to approximate a collapsed Gibbs sampler whose transition of $\lambda_v$'s are their marginal distributions.
Note that bridge regression is reduced to Lasso if $\alpha=1$ and a Gibbs sampler is feasible by imposing a Laplacian prior on $\beta$ \citep{park2008bayesian}. When $\alpha\neq 1$, \citet{polson2014bayesian} has proposed a Gibbs sampler that requires truncated multivariate distributions for parameter updates, which may require inefficient rejection sampling. Additionally, the data augmentation by the positive stable distributed variables that results in full conditional distributions of $\beta$ and $\sigma^2$ in  \eqref{eq:bridge_gibbs} is different from the one in \citet{polson2014bayesian} and cannot be reduced to the one in \citet{park2008bayesian} when $\alpha=1$. 

With $\alpha=0.5$, $1$, and $1.5$ we showcase MIVI for bridge regression on \texttt{diabetes} data \citep{efron2004least} 
and the validity of the extrapolated Markov chain \eqref{eq:bridge_gibbs} as an MCMC scheme. Since choosing the hyper-parameter $\rho$ is outside our scope of research, for $\alpha=1$ we run the Gibbs sampling \citep{park2008bayesian} with the suggested value of $\rho$, followed by MIVI ($T=3$) and frequentist Lasso that approximately match the $L_1$ norm of $\beta$. For $\alpha=0.5$ and $1.5$, we first use 4-fold cross-validation to select the value of the hyper-parameter, and then run MIVI ($T=3$) with $\rho$ chosen to match the $L_\alpha$ norm. In addition, we use the optimized Weibull distribution to extrapolate the Markov chain from random initial values.
In Figure \ref{fig:bridge} we provide the point estimates of $\beta$ resulting from MIVI, Gibbs sampling ($\alpha=1$) or the extrapolated Markov chain \eqref{eq:bridge_gibbs} ($\alpha=0.5$ and $1.5$) where MCMC samples from the last 1,000 of a total of 5,000 iterations are collected for inference, and the frequentist bridge regression along with ordinary least squares (OLS). Also reported are the $95\%$ confidence intervals (CIs) by MIVI and the Gibbs sampling or the extrapolated chain. 
While estimation of $\beta$ by $\tilde q_{\eta,\phi}$ of MIVI is not sparse in the exact sense, the point estimates (and CIs) coincide with those by frequentist Lasso (and Gibbs sampling for $\alpha=1$). Moreover, for $\alpha=0.5$ and $1.5$, frequentist estimates lie around the center of the CIs by MIVI and the extrapolated chains. For $\alpha=1$ we also run an extrapolated chain of MIVI and the CIs are similar to MIVI. Together with Sections \ref{sec:logistic}, the results endorse MIVI as an alternative to Gibbs sampling but in a way of simplicity and high efficiency. 
\vspace{-2mm}
\subsection{Variational autoencoders}\label{sec:vae}\vspace{-1mm}

We consider MIVI of latent variables in VAEs on two data sets. One is stochastically binarized MNIST \citep{salakhutdinov2008quantitative} consisting of 50,000 training and 10,000 testing images of hand written digits. The other is fashion MNIST (fMNIST) \citep{xiao2017fashion} consisting of 60,000 training and 10,000 testing images of clothing items, where the pixels are binarized at threshold $0.5$. The variational distribution $q_{\phi}(\zv\given \xv)$ of the latent code $\zv$ is diagonal Gaussian whose mean and log of variances are parameterized by two separate fully connected neural networks with two hidden layers of 200 units and ReLU activation functions. 
The same network structure is used for the Bernoulli probability of decoder $p_{\theta}(\xv\given \zv)$, except for a sigmoid transformation of the output. $T=5$ SGLD transitions are incorporated in MIVI, with $\eta$ as the parameter of a neural network whose input is $\xv_i$ and output is the time-invariant step size of the SGLD for $t=1,\ldots,T$.
For comparison, vanilla VAE \citep{kingma2013auto} and five recently proposed algorithms are used as benchmarks: semi-implicit VI (SIVI) \citep{yin2018semi}, doubly semi-implicit VI (DSIVI) \citep{molchanov2018doubly}, unbiased implicit VI (UIVI) \citep{titsias2018unbiased}, variational contrastive divergence (VCD) \citep{ruiz2019contrastive}, and variationally inferred sampling (VIS) \citep{gallego2019variationally}. SIVI, DSIVI and UIVI use implicit distributions as the variational distribution to provide a high degree of flexibility. VCD and VIS have been discussed in Section \ref{sec:relatedwork}. We reproduce these approaches with the same configuration and neural network structures as of MIVI.  Note that VCD has been reported to outperform  \citet{hoffman2017learning}, and the latter outperforms \citet{salimans2015markov} that uses Hamiltonian flow (see  \citet{ruiz2019contrastive} and \citet{hoffman2017learning}). 

We evaluate the performance via the average marginal log-likelihood calculated by importance sampling, written as 
$\textstyle 
\log p(\tilde \xv)\approx \log\frac{1}{{\tilde J}}\sum_{j=1}^{\tilde J} \frac{p_{\theta}(\tilde \xv \given \zv_j)p(\zv_j)}{\tilde{q}_{\eta,\phi}(\zv_j)}
$ for reasonably large ${\tilde J}$. See Appendix for detailed settings and discussion. 
For MIVI with $T=5$, we run 0 or 5 SGLD transitions with the optimized step sizes on testing images, denoted respectively by MIVI-5-0 that uses $q_{\phi}$ and MIVI-5-5 that uses ${\textstyle \tilde q^{(5)}_{\eta,\phi}}$ for testing. VIS also has 5 steps of SGLD for training and 5 for testing (denoted by VIS-5-5).   Provided in Table \ref{tab:vae40} is the performance comparison of the VAE algorithms for $\zv\in\mathbb{R}^{40}$. MIVI slightly outperforms other algorithms except VCD on MNIST and outperforms all the others on fMNIST. MIVI can be better than VIS because we use a neural network whose input is $\xv_i$ to learn the SGLD step size of $\zv_i$ for each $i$, whereas VIS learns (or pre-specifies) an equal step size for all $\zv_i$'s. Additional results on VAEs are provided in Appendix.

\begin{table}[!t]\vspace{-2mm}
\centering
\caption{\small Comparison of VAE algorithms on MNIST and fMNIST (${\textstyle \zv \in \mathbb{R}^{40}}$).}\label{tab:vae40} \vspace{-2mm}
\makebox[\linewidth]{
\resizebox{\linewidth}{!}{
\begin{tabular}{lcccccccccc}
  \toprule
 &Vanilla
 & SIVI 
 & DSIVI 
 & UIVI  
 & VCD 
 & VIS-5-5
 & MIVI-5-0 
 & MIVI-5-5\\
  \hline
  MNIST 
    & -86.48 & -84.71  & -83.79  & -83.47  & {\bf -81.01} & -83.82     &     -84.39     & -83.09       \\
 fMNIST 
    &  -121.95  & -118.69 & -112.02 & -109.97 & -109.90 & -106.96    &      -108.86    &  {\bf -102.51}     \\
\bottomrule
\end{tabular}
}}\vspace{-5mm}
\end{table}

\section{Conclusion}\label{sec:conclusion}\vspace{-1.5mm}
The proposed MIVI incorporating a short Markov chain encourages VI and MCMC to overcome each other's limitations and to achieve mutual improvement.
We establish MIVI by auxiliary optimizations so that all the gradients can be rigorously computed and the training becomes stable.
We formulate the Markov chain by transition functions that are partly adopted from valid MCMC 
and partly optimized in the framework of VI. Moreover, we prove the short chain in MIVI can be extrapolated and serve as an efficient MCMC that approaches towards the posterior, and consequently, MIVI facilitates designs of MCMC transitions.
Therefore, capable of posterior approximation and simulation without keeping track of an MCMC trajectory, MIVI is an overall solution to
effective and efficient point estimation and uncertainty quantification.



\newpage
\bibliography{mivi.bib}
\bibliographystyle{iclr2021_conference}

\newpage
\appendix
\begin{center}
\Large{\textbf{MCMC-Interactive Variational Inference: Appendix}}\\\vspace{2mm}
\end{center}
\section{Proofs}
\begin{proof}[Proof of Property \ref{prop:gradient}]
Calculating $\nabla_{\eta}\mathbb{E}_{\tilde{q}_{\eta,\phi^*}(\zv)}\left[\log{p_{\theta}(\xv, \zv)}-\log{q_{\phi^*}(\zv)}\right]$ is straightforward and hence we only need to derive $\nabla_{\eta}\mathbb{E}_{\tilde{q}_{\eta,\phi^*}(\zv)}D^*(\xv, \zv)$. 
Given $\phi^*$, $D^*(\xv,\zv)=\log \frac{\tilde{q}_{\eta,\phi^*}(\zv)}{q_{\phi^*}(\zv)}$, and the fact that the expectation of a score function is $0$, we have
\begin{align*}
\mathbb{E}_{\tilde q_{\eta, \phi^*}(\zv)}\nabla_{\eta}D^*(\xv,\zv) = \mathbb{E}_{\tilde q_{\eta, \phi^*}(\zv)}\nabla_{\eta}\log \tilde{q}_{\eta,\phi^*}(\zv) =0.
\end{align*}
Consequently, with a reparameterizable $\tilde{q}_{\eta,\phi}$ we have 
\begin{align*}
\mathbb{E}_{ \varepsilon}\left[(\nabla_\eta D^*)(\xv, f_\eta(\zv^{(0)},  \varepsilon)) \right]=0. 
\end{align*}
Therefore, taking the gradient of $\mathbb{E}_{\tilde{q}_{\eta,\phi^*}(\zv)} D^*(\xv, \zv)$ with respect to $\eta$ we get
\begin{align*}
\nabla_{\eta}\mathbb{E}_{\tilde{q}_{\eta,\phi^*}(\zv)} D^*(\xv, \zv) &=
\nabla_{\eta}\mathbb{E}_{  \varepsilon} D^*(\xv, f_\eta(\zv^{(0)},  \varepsilon))\\
&=\mathbb{E}_{  \varepsilon} \left[\nabla_{\eta}D^*(\xv, f_\eta(\zv^{(0)},  \varepsilon))\right]\\
&=\mathbb{E}_{  \varepsilon} \left[ (\nabla_{\eta}D^*)(\xv, f_\eta(\zv^{(0)},  \varepsilon)) +  
	(\nabla_{\eta}f_\eta(\zv^{(0)},  \varepsilon ) )  (\frac{dD^*(\xv, \zv)}{d\zv}\,\Big |\,_{\zv=f_\eta(\zv^{(0)}, \varepsilon )}) \right]\\
&=\mathbb{E}_{  \varepsilon} \left[(\nabla_{\eta}f_\eta(\zv^{(0)}, \varepsilon ))  (\frac{dD^*(\xv, \zv)}{d\zv}\,\Big |\,_{\zv=f_\eta(\zv^{(0)},  \varepsilon )}) \right].
\end{align*}  
\end{proof}

\begin{proof}[Proof of Property \ref{prop:bound}]
$$\mathbb{E}_{\tilde{q}_{\eta,\phi}(\zv)}\log\frac{p(\xv, \zv)}{\tilde{q}_{\eta,\phi}(\zv)} = \mathbb{E}_{\tilde{q}_{\eta,\phi}(\zv)}\log\frac{p(\xv, \zv)}{ q_{\phi}(\zv)} - \mbox{KL}(\tilde{q}_{\eta,\phi}(\zv)\,||\, q_{\phi}(\zv))
\leq \mathbb{E}_{\tilde{q}_{\eta,\phi}(\zv)}\log\frac{p(\xv, \zv)}{ q_{\phi}(\zv)}.
$$
\end{proof}


\begin{lem}[Chain rule of KL divergence, page 25 of \citet{cover2006elements}]\label{lemma:chainrule} 
\begin{align*}
\emph{\mbox{KL}}(q(w,z)\,||\,q'(w,z)) = \emph{\mbox{KL}}(q(w)\,||\,q'(w))+\mathbb{E}_{q(w)}\emph{\mbox{KL}}(q(z\given w)\,||\,q'(z\given w)).    
\end{align*}
\end{lem}

\begin{proof}[Proof of Lemma \ref{thm:gibbs}]
Let $\mu'^{(t)}$ be an arbitrary distribution on the state space of $M$ at $t$ with the joint mass/density function denoted by $q'$ and thus $q'(\wv^{(t)},\zv^{(t)},  \zv^{(t+1)}) = q'(\wv^{(t)},\zv^{(t)})r( \zv^{(t+1)}\given \wv^{(t)},\zv^{(t)})$. By Lemma \ref{lemma:chainrule},
\begin{align*}
&\mbox{KL}(q(\wv^{(t)},\zv^{(t)},  \zv^{(t+1)}\,||\, q'(\wv^{(t)},\zv^{(t)},  \zv^{(t+1)}))  \\
= & \mbox{KL}(q(\wv^{(t)},\zv^{(t)})\,||\, q'(\wv^{(t)},\zv^{(t)})) +\mathbb{E}_{q(\wv^{(t)},\zv^{(t)})}\mbox{KL}(q(\zv^{(t+1)}\given \wv^{(t)},\zv^{(t)})\,||\, q'(\zv^{(t+1)} \given \wv^{(t)},\zv^{(t)}))  \\
= & \mbox{KL}(q(\wv^{(t)},\zv^{(t)} )\,||\, q'(\wv^{(t)},\zv^{(t)})) +\mathbb{E}_{q(\wv^{(t)},\zv^{(t)})}\mbox{KL}(r(\zv^{(t+1)}\given \wv^{(t)}) \,||\, r(\zv^{(t+1)} \given \wv^{(t)})) \\
= & \mbox{KL}(q(\wv^{(t)},\zv^{(t)})\,||\, q'(\wv^{(t)},\zv^{(t)}))
\end{align*}
Again, by Lemma \ref{lemma:chainrule},
\begin{align*}
 &\mbox{KL}(q(\wv^{(t)},\zv^{(t)},\zv^{(t+1)})\,||\, q'(\wv^{(t)},\zv^{(t)},\zv^{(t+1)}))  \\
= & \mbox{KL}(q(\wv^{(t)},\zv^{(t+1)})\,||\, q'(\wv^{(t)},\zv^{(t+1)})) +\mathbb{E}_{q(\wv^{(t)},\zv^{(t+1)})}\mbox{KL}(q(\zv^{(t)}\given \wv^{(t)},\zv^{(t+1)}) \,||\, q'(\zv^{(t)} \given \wv^{(t)},\zv^{(t+1)}))  \\
\geq & \mbox{KL}(q(\wv^{(t)},\zv^{(t+1)} )\,||\, q'(\wv^{(t)},\zv^{(t+1)} ))
\end{align*}
So $\mbox{KL}(q(\wv^{(t)},\zv^{(t)} )\,||\, q'(\wv^{(t)},\zv^{(t)} ))
\geq \mbox{KL}(q(\wv^{(t)},\zv^{(t+1)} )\,||\, q'(\wv^{(t)},\zv^{(t+1)} ))$. Let $q'(\wv^{(t)},\zv^{(t)}) = \pi(\wv,\zv)$, i.e., $\mu'_t$ is the posterior distribution to which the Gibbs sampler $G$ converges. Since $M$'s transition  $r(\zv|\wv)$ is the conditional distribution as well as the transition of $G$, $q'(\wv^{(t)},\zv^{(t+1)}) = q'(\wv^{(t)})r(\zv^{(t+1)}\given \wv^{(t)})=\pi(\wv,\zv)$. Therefore, 
\begin{align*}
    \mbox{KL}(q(\wv^{(t)},\zv^{(t)}))\,||\, \pi(\wv,\zv))\geq \mbox{KL}(q(\wv^{(t)},\zv^{(t+1)})\,||\, \pi(\wv,\zv)).
\end{align*}
\end{proof}

\begin{proof}[Proof of Lemma \ref{cor:collapsed}]
Since $\mu^{(t)}$ can be any distribution, we assume that $q(\wv^{(t)},\zv^{(t)}) = q(\zv^{(t)})r(\wv^{(t)}\given \zv^{(t)})$ for arbitrary $q(\zv^{(t)})$.
By the proof of Lemma \ref{thm:gibbs},
\begin{align*}
    &\mbox{KL}(q( \wv^{(t)},\zv^{(t)})\,||\,q'( \wv^{(t)},\zv^{(t)})) \geq \mbox{KL}(q(\wv^{(t)},\zv^{(t+1)})\,||\,q'( \wv^{(t)},\zv^{(t+1)})). 
\end{align*}
By Lemma \ref{lemma:chainrule}, 
The left-hand side of this inequality is equal to 
\begin{align*}
    \mbox{KL}(q( \zv^{(t)})\,||\,\pi(\zv)) +  \mathbb{E}_{q(\zv^{(t)})}\mbox{KL}(r(\wv \given \zv^{(t)} ) \,||\, \pi(\wv\given \zv^{(t)})),
\end{align*}
and the right-hand side is equal to 
\begin{align*}
    \mbox{KL}(q( \zv^{(t+1)})\,||\,q'( \zv^{(t+1)})) + \mathbb{E}_{q(\zv^{(t+1)})}\mbox{KL}(q(\wv^{(t)}\given \zv^{(t+1)}) \,||\,q'(\wv^{(t)}\given \zv^{(t+1)})).
\end{align*}
Since $\mathbb{E}_{q(\zv^{(t)})}\mbox{KL}(r(\wv\given \zv^{(t)}) \,||\, \pi(\wv\given \zv^{(t)}))\leq \mathbb{E}_{q(\zv^{(t+1)})}\mbox{KL}(q(\wv^{(t)}\given \zv^{(t+1)}) \,||\, q'(\wv^{(t)}\given \zv^{(t+1)}))$, we have
\begin{align*}
    \mbox{KL}(q( \zv^{(t)})\,||\,\pi(\zv)) \geq  \mbox{KL}(q( \zv^{(t+1)})\,||\,q'( \zv^{(t+1)})).
\end{align*}
Considering $q'( \zv^{(t+1)}) = \int r(\zv^{(t+1)}\given \wv^{(t)})q'(\wv^{(t)})d \wv^{(t)}$ and  $r(\zv\given \wv)$ is the conditional distribution of $\zv$ given $\wv$, we have $q'( \zv^{(t+1)}) = q'(\zv^{(t)})=\pi(\zv)$. Therefore,
\begin{align*}
    \mbox{KL}(q( \zv^{(t)})\,||\, \pi(\zv))\geq \mbox{KL}(q(\zv^{(t+1)})\,||\, \pi(\zv)).
\end{align*}
\end{proof}

\section{Full algorithm and detailed implementation}
We summarize the implementation of MIVI as Algorithm \ref{algorithm}. The neural network structure of discriminator $D$ depends on the dimension of $\bf z$ and does not have to be complex because, anyway, the sigmoid function saturates when $q_\phi$ and $\tilde q_{\eta,\phi}$ are far from each other at an early stage of training. To avoid a potentially time-consuming optimization of $D$, we simply omit it at the early stage of training according to the analysis of Property \ref{prop:bound}, and start training $D$ after first $M$ epochs when $q_\phi$ and $\tilde q_{\eta,\phi}$ get closer. In this way, a reasonably flexible $D$ is good enough. Since a cross-entropy loss is used to train $D$, $D$ works well if the loss drops from a large value towards 0, which have been observed in our experiments.
Furthermore, we find that the algorithm converges faster if we stop the gradient of $\zv_j^{(t)}$ with respect to $\phi$ in line 10 of Algorithm \ref{algorithm}; in PyTorch, we use the command \texttt{.detach()} on $\zv_j^{(t)}$. Essentially, stopping the gradient is only optional and does not change parameter estimations in our experiments. We minimize \eqref{eq:xentropy} where the expectation is approximated by sampling $\zv^{(t)}$ from $\tilde q_{\eta,\phi}$ to let $q_\phi$ and $\tilde q_{\eta,\phi}$ get close to each other; stopping the gradient of $\zv^{(t)}$ with respect to $\phi$ can be regarded as fixing $\tilde q_{\eta,\phi}$. In this way, we let $q_\phi$ approach to $\tilde q_{\eta,\phi}$ that has been well learned by optimizing \eqref{eq:obj_lrbnd} so faster convergence can be achieved. Note that $\tilde q_{\eta,\phi}$ is less and less dependent on $\phi$ as the number of transitions increases. So there is no need to stop the gradient if $T$ is large.

\begin{algorithm}
    \caption{\small MCMC-interactive variational inference}\label{algorithm}
  \begin{algorithmic}[1]
    \INPUT  Data $\xv$, model $p_{\theta}(\xv\given \zv)$, prior $p(\zv)$, reparameterizable variational family $q_{\phi}$ and Markov chain updating function $f_{\eta}$ implied by reparameterizable $h_\eta$
    \OUTPUT $\theta, \phi, \eta$
    \STATE $\mbox{Epoch}\leftarrow 0$.
    \WHILE{not converge}
      \STATE Draw mini-batch $\xv$ of size $n$ from training data of Size $N$ (for MIVI with SGLD)
      \STATE Sample $\zv_j^{(0)} \overset{iid}{\sim} q_{\phi}(\zv)$, $j=1,\ldots,J$. 
      \STATE\# Begin Markov chain transitions:
      \FOR{$t=1,\cdots,T$, say $T=3$ and $j=1,\ldots,J$} 
      	\STATE  $\zv_j^{(t)}=f_{\eta}(\zv_j^{(t-1)}, \varepsilon_j^{(t)})$ with some independent random vector $\varepsilon_j^{(t)}$.\label{eq:mcmc_step}
      \ENDFOR
      \STATE\# Begin optimization:
      \STATE Update $\phi$ by descending the gradient  
      				$$-\nabla_{\phi}\frac{1}{J\times T}\sum_{j,t}\log q_{\phi}(\zv_j^{(t)})$$
      \IF{$\mbox{Epoch}<M$, say $M=100$}
      	 \STATE Update $\theta$ and $\eta$ by ascending the gradient
      				$$\nabla_{\theta,\eta} \frac{1}{J\times T}\sum_{j,t}\left[\log\frac{p_{\theta}(\xv \given \zv_j^{(t)})p(\zv_j^{(t)})}{ q_{\phi}(\zv_j^{(t)})} \right]$$
      \ELSE
        \STATE Update $\theta$ and $\eta$ by ascending the gradient
      				$$\nabla_{\theta,\eta} \frac{1}{J\times T}\sum_{j,t}\left[\log\frac{p_{\theta}(\xv \given \zv_j^{(t)})p(\zv_j^{(t)})}{ q_{\phi}(\zv_j^{(t)})} - D(\zv_j^{(t)})\right]$$
        \STATE Update $D$ by maximizing 
      				$$\frac{1}{J\times T}\sum_{j,t}\log\sigma(D(\zv_j^{(t)})) +\frac{1}{J}\sum_{j}\log\left(1-\sigma(D(\zv_j^{(0)}))\right)$$
      \ENDIF				
      \STATE  $\mbox{Epoch}\leftarrow \mbox{Epoch}+1$
    \ENDWHILE
  \end{algorithmic}
\end{algorithm}

\section{Bayesian logistic regression and Polya gamma distribution}
For a unique $\xv_i$, $i=1,\cdots,n$ and $y_i\in\{0,1\}$, the hierarchical model for Bayesian logistic regression can be expressed as 
\begin{align*}
y_i &\sim \mbox{Bernoulli}(1/(1+e^{-\xv_i'\beta})),\\
\beta &\sim \mathcal{N}(b,B).
\end{align*}
As in  \citet{polson2013bayesian}, under the Polya-gamma (PG) distribution based data augmentation,
the full conditional distributions can be expressed as 
\begin{align*}
(\omega_i\given-) &\sim \mbox{PG}(1, \xv_i'\beta),~i=1,\ldots,n,\\
(\beta\given -) &\sim \mathcal{N}\left(\Sigma(X'\kappa + B^{-1}b), \Sigma\right),
\end{align*}
where $\Sigma = (\bm X'\Omega \bm X  +B)^{-1}$, $\Omega=\mbox{diag}(\omega_1,\cdots, \omega_n)$ and $\kappa = (y_1-\frac{1}{2},\cdots,y_n-\frac{1}{2})$. Note that $\omega_i\sim \mbox{PG}(1,\xv_i'\beta)$ is equivalent to $\omega_i{=}\frac{1}{2\pi}\sum_{k=1}^{\infty}\frac{\gamma_k}{(k-1/2)^2+(\xv_i'\beta/2\pi)^2}$ where $\gamma_k\overset{iid}{\sim}\mbox{Gamma}(1,1)$. 
To generate PG random variables, \citet{polson2013bayesian} use rejection sampling with finite truncations of this expression as the proposal and \citet{zhou2012lognormal} uses finite truncations together with matching the first- and second-order moments. Neither solution, however, can be used as a Markov chain transition in MIVI due to the lack of reparameterization.

We plot the estimated posteriors of $\omega_i$'s associated to the synthesized data by Gibbs sampling and MIVI in Figure \ref{fig:pg_toy} and those of binary MNIST by MIVI in Figure \ref{fig:pg_mnist} for eight randomly selected $i$ in training data.
\begin{figure}[!ht]\vspace{0mm}
 \centering
\includegraphics[width=0.8\linewidth]{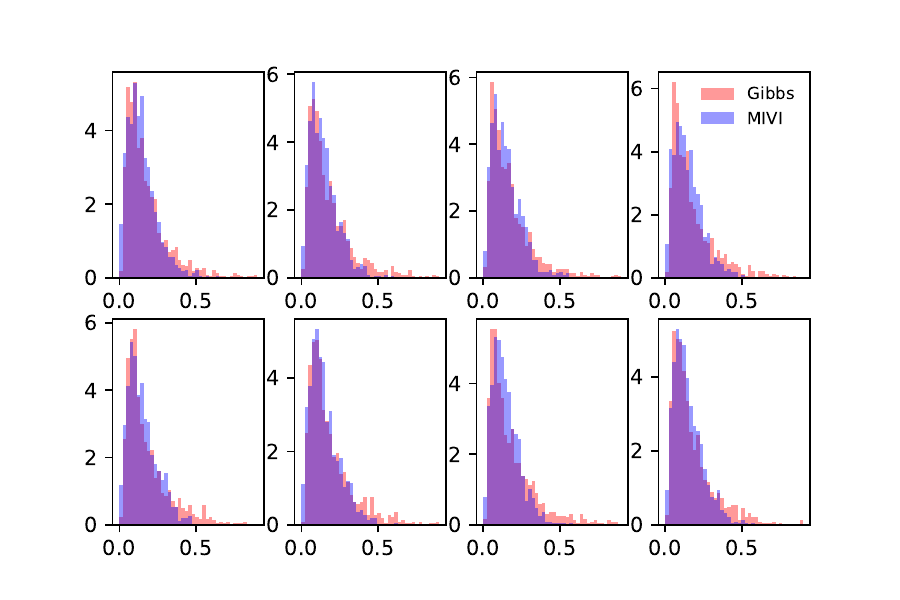}\vspace{-2.5mm} 
 \caption{\small PG auxiliary variable $\omega_i$ by Gibbs sampling (red) and MIVI (blue) for eight randomly selected samples of the synthesized data in Section \ref{sec:logistic}.}\label{fig:pg_toy}\vspace{-1mm}
 \end{figure}%

\begin{figure}[!ht]
 \centering
\includegraphics[width=0.8\linewidth]{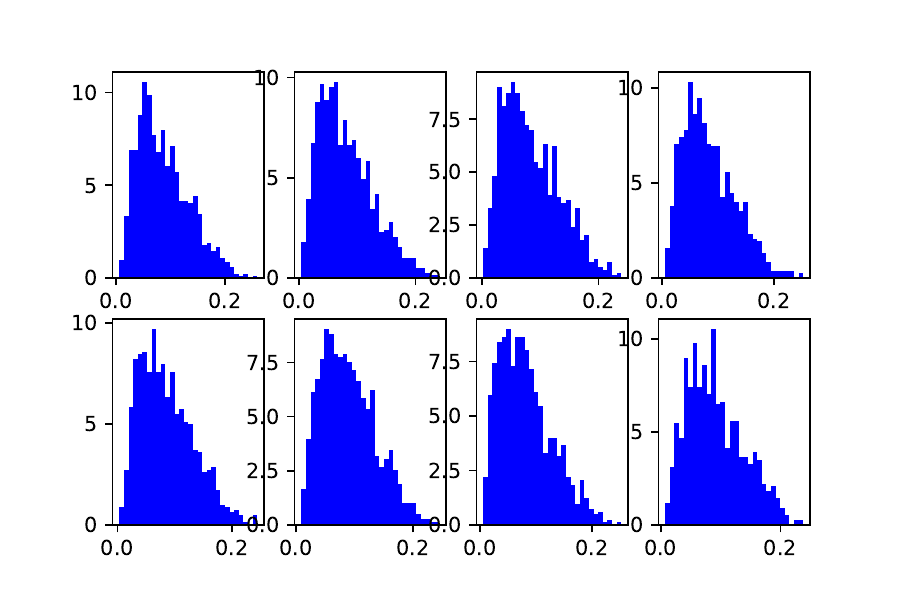}\vspace{-2.5mm} 
 \caption{\small PG auxiliary variable $\omega_i$ by MIVI for eight randomly selected training images of the binary MNIST data in Section \ref{sec:logistic}.}\label{fig:pg_mnist}\vspace{-1mm}
 \end{figure}%

\section{Experiment settings}\label{sec:settings}
\subsection{General settings}
With the definitions of $J$ and $M$ in Algorithm \ref{algorithm}, 
we run 1,000 epochs with $J=200$, $T=5$, and $M=100$ in the toy experiment of Section \ref{sec:toy} 
and 2,000 epochs with $J=1000$ and $M=0$ for the negative binomial model in Section \ref{sec:NB}. For the Bayesian logistic in Section \ref{sec:logistic} we run 1,000 epochs with $J=200$ and $M=0$. For the Bayesian bridge regression in Section \ref{sec:bridge} we run 1,000 epochs with $J=100$ and $M=0$. For the VAE by MIVI in Section \ref{sec:vae} we run 2,500 epochs with $J=10$ and $M=200$.

For experiments of VAE on MNIST and FashionMNIST, we follow the original partition to split the data as 50,000/10,000/10,000 for training/validation/test. The MNIST data is dynamically binarized, and the FashionMNIST data is binarized with 0.5 as a threshold for each pixel. The dimension of the latent variable $\zv$ is set as 40.  
To ensure the fairness of comparison, we use the same network architecture to build up the VAE on UIVI and VCD and use the same experiment configuration as in \citet{titsias2018unbiased} and \citet{ruiz2019contrastive}. We apply a 2-hidden-layer network with 200 hidden units for both encoder and decoder and choose \textit{ReLU} as the activation function. Then we optimize the model using the initial learning rate as $0.001$ with a $10\%$ decay for every 15,000 iterations, and choose the best model with validation set for testing. Specifically for SIVI-VAE and DSIVI-VAE, the dimension of $\psi$ is set as 500. For MIVI we run $2,500$ epochs with the initial Adam learning rate as $0.001$ (with a $12\%$ decay for every 100 epochs) for MNIST and $0.0001$ (with a $10\%$ decay for every 200 epochs) for fMNIST.

\subsection{Performance evaluation of MIVI on VAEs}
In Section \ref{sec:vae} we evaluate MIVI for VAEs by estimating the average marginal log-likelihood,
\begin{align*}
\log p(\tilde \xv)\approx \log\frac{1}{{\tilde J}}\sum_{j=1}^{\tilde J} \frac{p_{\theta}(\tilde \xv \given \zv_j)p(\zv_j)}{\tilde{q}_{\eta,\phi}(\zv_j)} 
=\log\frac{1}{{\tilde J}}\sum_{j=1}^{\tilde J} \frac{p_\theta(\tilde \xv \given \zv_j)p(\zv_j)}{q_{\phi}(\zv_j\given \tilde{\xv})} e^{-D^*(\tilde \xv, \zv_j)}. 
\end{align*}
The correctness of right-hand side of this equation depends on an optimal discriminator $D^*$, which can be hard to verify. Therefore, we use the Gaussianity of SGLD and a Monte Carlo method to evaluate $\tilde q_{\eta,\phi}$. Specifically,  the updating function of SGLD is $f_\eta(\zv, \epsilon)$ such that
\begin{align*}
\zv^{(t)} &= f_{\eta_t}(\zv^{(t-1)}, \epsilon_t)\\
&=
\zv^{(t-1)}+
 \frac{\eta_t}{2}\odot [\nabla_{\zv} \log p(\zv^{(t-1)}) +\frac{N}{n} \nabla_{\zv} \log p(\xv \given \zv^{(t-1)})] + \epsilon_t 
\end{align*}
where $\epsilon_t \sim \mathcal{N}(\bm 0,\mbox{diag}(\eta_t))$ and $\odot$ stands for element-wise multiplication. So 
we have $\zv^{(T)}\sim \mathcal{N}(\mu(\zv^{(T-1)}, \eta_T),\eta_T)$ where $\mu(\zv, \eta) = \zv +
 \frac{\eta }{2}\odot[\nabla_{\zv} \log p(\zv ) +\frac{N}{n} \nabla_{\zv} \log p(\xv \given \zv )]$ and consequently, 
\begin{align*}
\zv^{(T)}&\sim \mathcal{N}\left(\mu(\zv^{(T-1)}, \eta_T),\mbox{diag}(\eta_T)\right)\\
& =  \mathcal{N}\left(\mu( f_{\eta_{T-1}}(\zv^{(T-2)}, \epsilon_{T-1}), \eta_{T-1}),\mbox{diag}(\eta_T)\right)\\
&= \mathcal{N}\left(\mu( f_{\eta_{T-1}}(f_{\eta_{T-2}}(\zv^{(T-3)}, \epsilon_{T-2}), \epsilon_{T-1}), \eta_{T-1}),\mbox{diag}(\eta_T)\right)\\
&= \mathcal{N}\left(\mu( f_{\eta_{T-1}}(f_{\eta_{T-2}}(\ldots(f_{\eta_{1}}(\zv^{(0)}, \epsilon_{1}),\epsilon_{2})\ldots), \epsilon_{T-1}),\mbox{diag}(\eta_T)\right).
\end{align*}
Therefore, the marginal distribution $\tilde q_{\eta,\phi}$ is equal to 
\begin{align}
&  \int \ldots \int \mathcal{N}\left(\mu( f_{\eta_{T-1}}(f_{\eta_{T-2}}(\ldots(f_{\eta_{1}}(\zv^{(0)}, \epsilon_{1}),\epsilon_{2})\ldots), \epsilon_{T-1})),\mbox{diag}(\eta_T)\right) dP(\epsilon_1) \ldots dP(\epsilon_{T-1})q_\phi(\zv^{(0)})d\zv^{(0)}\nonumber\\
 \approx & \frac{1}{K}\sum_{k=1}^{K} \mathcal{N}\left(\mu( f_{\eta_{T-1}}(f_{\eta_{T-2}}(\ldots(f_{\eta_{1}}(\zv^{(0)}_k, \epsilon_{1,k}),\epsilon_{2,k})\ldots), \epsilon_{T-1,k})),\mbox{diag}(\eta_T)\right)\label{eq:sgld_marginal}
\end{align}
where $\zv_k^{(0)}\overset{iid}{\sim} q_\phi$, $\epsilon_{t,k}\overset{ind}{\sim}\mathcal{N}(0,\mbox{diag}(\eta_t))$ for  $k=1,\ldots,K$ and $t=1,\ldots,T-1$. We evaluate the performance of MIVI for VAEs by 
\begin{align}
\log p(\tilde \xv)&\approx \log\frac{1}{{\tilde J}}\sum_{j=1}^{\tilde J} \frac{p_{\theta}(\tilde \xv \given \zv_j)p(\zv_j)}{\hat {q}_{\eta,\phi}(\zv_j)} \label{eq:importance_sampling}
\end{align}
where $\zv_j = f_{\eta_{T }}(f_{\eta_{T-1}}(\ldots(f_{\eta_{1}}(\zv^{(0)}_j, \epsilon_{1,j}),\epsilon_{2,j})\ldots), \epsilon_{T,j})$, $\zv^{(0)}_j\overset{iid}{\sim} q_\phi$ and 
\begin{align}
\hat{q}_{\eta,\phi}(\zv_j) = &\frac{1}{K+1}  \mathcal{N}\left(\zv_j\given\mu( f_{\eta_{T-1}}(f_{\eta_{T-2}}(\ldots(f_{\eta_{1}}(\zv^{(0)}_j, \epsilon_{1,j}),\epsilon_{2,j})\ldots), \epsilon_{T-1,j})),\eta_T\right)+\nonumber\\
&\frac{1}{K+1} \sum_{k=1}^K  \mathcal{N}\left(\zv_j\given \mu( f_{\eta_{T-1}}(f_{\eta_{T-2}}(\ldots(f_{\eta_{1}}(\zv^{(0)}_k, \epsilon_{1,k}),\epsilon_{2,k})\ldots), \epsilon_{T-1,k})),\eta_T\right) \label{eq:marginal_sgld}
\end{align}
analogous to \citet{yin2018semi}.

We set ${\tilde J}=1000$ and $K=50$ for the evaluation by the importance sampling. Note what we are estimating in \eqref{eq:importance_sampling} 
is in fact a lower bound of $\log p(\tilde \xv)$ \citep{burda2015importance}. Its quality depends on both the decoder $p_\theta(\xv\given \zv)$ and the encoder which is used as the importance distribution; fixing $p_\theta(\xv\given \zv)$, a poor importance distribution may give rise to a loose bound. The estimation of MIVI-5-5 (using $\tilde q^{(T)}_{\eta,\phi}$ as the importance distribution) is better than that of MIVI-5-0 (using $q_\phi$ as the importance distribution) because $p_\theta(\xv\given \zv)$ is trained based on $\tilde q^{(T)}_{\eta,\phi}$. 
Moreover, in case of multimodality of $p(\zv\given \xv)$ which is very probable for VAE models, $q_\phi$ can be lighter-tailed than $\tilde q_{\eta,\phi}$ and may result in larger variance of the importance sampling estimation.
In addition, we need be  careful about extrapolation when conducting the importance sampling based estimation as in~\eqref{eq:importance_sampling}, which is only valid under the assumption that the importance distribution $q$ satisfies $q(\zv)>0$ when $p(\xv\given \zv)p(\zv)\neq 0$ \citep{owen2009importance}. Concretely, though we have observed that the value obtained by \eqref{eq:importance_sampling} for MIVI-5-$t$ increases as $t$ grows, 
 that value may no longer reflect the true performance of the model, since  $\tilde{q}^{(t)}_{\eta,\phi}$ may no longer maintain non-negligible density on the regions where the joint likelihood has non-negligible values. So we only compare MIVI-5-5 so that the number of transitions are the same in training and testing.

\section{Supplimentary experimental results}
\subsection{Toy experiments}\label{sec:toy}
\begin{table}[!ht]
\centering
\caption{\small Target bivariate distributions.}\label{tab:bivariate}
\makebox[\linewidth]{
\resizebox{\linewidth}{!}{
\begin{tabular}{lll}
  \toprule
Correlated Gaussian & Banana & Gaussian mixture   \\ 
  \hline
$\mathcal{N}\left(
\begin{bmatrix}
0\\ 
0
\end{bmatrix},
\begin{bmatrix}
1 & 0.8\\ 
0.8 & 1
\end{bmatrix}\right)$ 
&  
$
\mathcal{N}(z_1;\frac{z_2^2}{4},1) 
\mathcal{N}(z_2;0,4)
$
& 
$\frac{1}{2}\mathcal{N}\left(
\begin{bmatrix}
-1\\ 
-1
\end{bmatrix},
\begin{bmatrix}
1 & -0.5\\ 
-0.5 & 1
\end{bmatrix}\right)+\frac{1}{2}\mathcal{N}\left(
\begin{bmatrix}
1.3\\ 
1.3
\end{bmatrix},
\begin{bmatrix}
1 & 0.3\\ 
0.3 & 1
\end{bmatrix}\right)$ \\ 
   \bottomrule
\end{tabular}
}}
\end{table}
\begin{figure}[!ht]\vspace{0mm}
 \centering
 \begin{subfigure}[t]{0.33\textwidth}
 \centering
\includegraphics[width=1\linewidth]{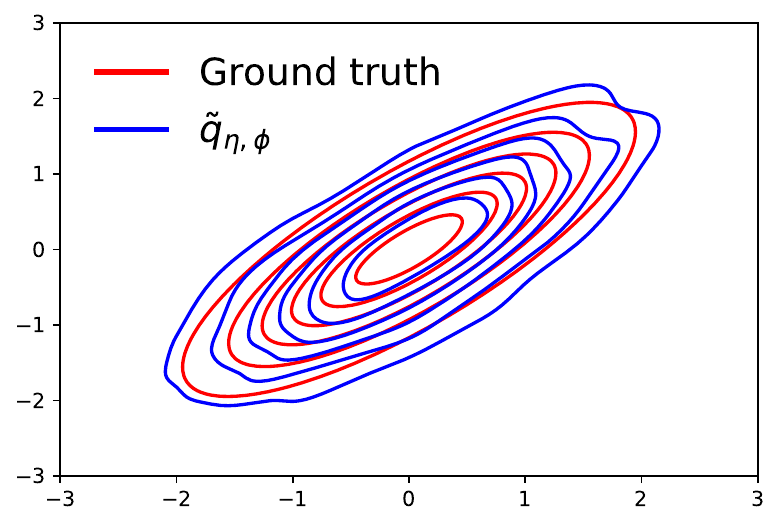}\vspace{-2.5mm} 
 \caption{\small Correlated Gaussian.}\vspace{-1mm}
 \end{subfigure}%
 \begin{subfigure}[t]{0.33\textwidth}
 \centering
\includegraphics[width=1\linewidth]{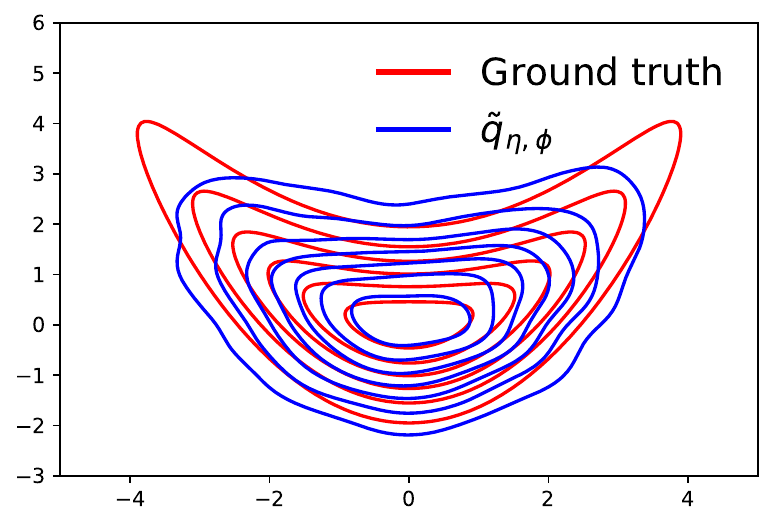}\vspace{-2.5mm} 
 \caption{\small Banana.}\vspace{-1mm}
 \end{subfigure}%
\begin{subfigure}[t]{0.33\textwidth}
 \centering
\includegraphics[width=1\linewidth]{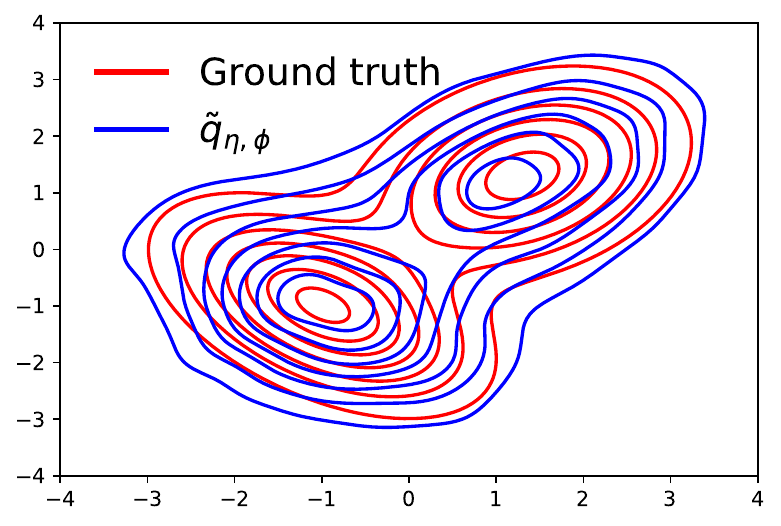}\vspace{-2.5mm} 
\caption{\small Gaussian mixture.}\vspace{-1mm}
 \end{subfigure}
\caption{\small Target distributions (red) and fitted $\tilde q_{\eta,\phi}$ (blue) of MIVI.
 }\label{fig:bivariate}\vspace{-1mm}
\end{figure}

\begin{figure}[!ht]\vspace{0mm}
 \centering
 \begin{subfigure}[t]{0.33\textwidth}
 \centering
\includegraphics[width=1\linewidth]{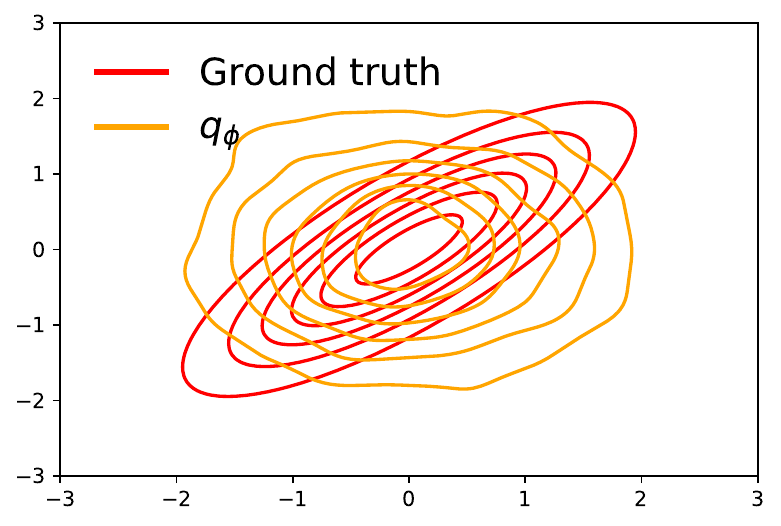}\vspace{-2.5mm} 
 \caption{\small Correlated Gaussian.}\vspace{-1mm}
 \end{subfigure}%
 \begin{subfigure}[t]{0.33\textwidth}
 \centering
\includegraphics[width=1\linewidth]{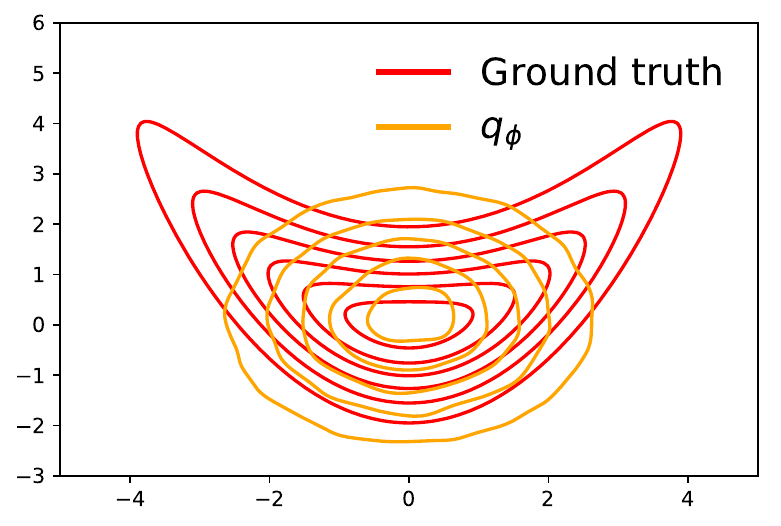}\vspace{-2.5mm} 
 \caption{\small Banana.}\vspace{-1mm}
 \end{subfigure}%
\begin{subfigure}[t]{0.33\textwidth}
 \centering
\includegraphics[width=1\linewidth]{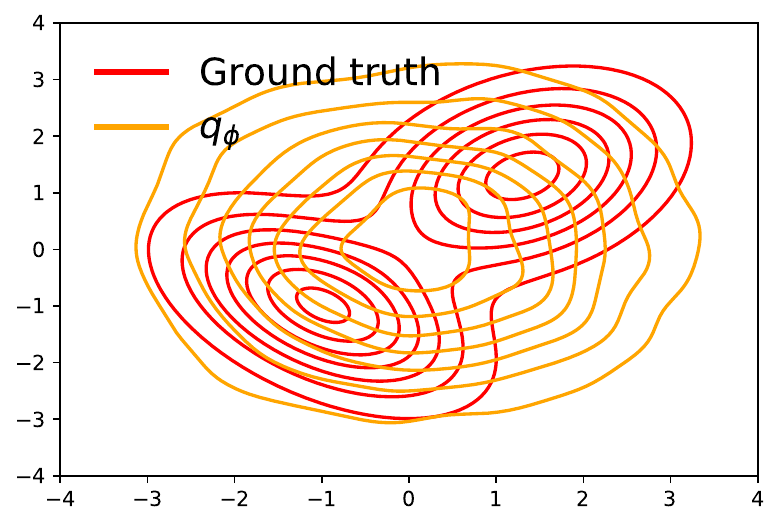}\vspace{-2.5mm} 
\caption{\small Gaussian mixture.}\vspace{-1mm}
 \end{subfigure}
\caption{\small Target distributions (red) and fitted $ q_{\phi}$ (orange) of MIVI.
 }\label{fig:bivariate_q0}\vspace{-1mm}
\end{figure}

To show the validity and flexibility of $\tilde{q}_{\phi,\eta}$ of SGLD in MIVI, we fit synthetic bivariate distributions listed in Table \ref{tab:bivariate}. 
Figure \ref{fig:bivariate} shows the contour plots of the synthetic bivariate distributions (red) along with the fitted $\tilde{q}_{\eta,\phi}(\zv)$ (blue). In all cases, $\tilde{q}_{\eta,\phi}(\zv)$ has well recovered the target distribution and captured the bivariate correlation, dependence, and multimodality, respectively, despite the small number of SGLD updates. In addition, Figure \ref{fig:bivariate_q0} shows that $q_\phi$ of MIVI has captured the large varianace of each dimension of $\zv$.


\subsection{Additional results of VAE}

\begin{table}[!ht]
\centering
\caption{\small Comparison of VAE algorithms on MNIST and fMNIST (${\textstyle \zv \in \mathbb{R}^{10}}$).}\label{tab:vae_z10_new} 
\makebox[\linewidth]{
\resizebox{\linewidth}{!}{
\begin{tabular}{lcccccccc}
  \toprule
 &Vanilla
 & SIVI 
 & DSIVI 
 & UIVI
 & VCD
 & VIS-5-5
 & MIVI-5-0 
 & MIVI-5-5\\
  \midrule
MNIST  
    &  -97.82   & -96.78  & {-89.96} & -94.09 & -95.86 & {\bf -87.65}   & -92.04 & -88.50  \\ 
  fMNIST  
    & -124.73 & -121.42 & {-121.39}  & {\bf -110.72} & -117.65 & -116.27  & -117.74  & -113.17   \\ 
\bottomrule
\end{tabular}
}}
\end{table}

We try a lower dimensional $\zv$, set $\zv\in \mathbb{R}^{10}$ in all models, keep other settings the same as in $\zv\in \mathbb{R}^{40}$, and report the VAE model comparison in Table~\ref{tab:vae_z10_new} where we cite the results of UIVI and VCD for $\zv\in\mathbb{R}^{10}$ from \citet{titsias2018unbiased} and \citet{ruiz2019contrastive}, respectively. It is shown that MIVI-5-5 is as good as VIS-5-5 which also uses SGLD for a refined encoder, and outperforms implicit VI approaches (except UIVI on fMNIST) because MIVI's encoder as in \eqref{eq:marginal_sgld} is not only flexible but also less complex in parameterization and hence easy to optimize.
We show reconstructions of randomly selected binarized MNIST testing images by MIVI in Figure \ref{fig:mnist_test} panel (a) and some of the most improved ones in panel (b). The first column is the testing image, the second column is the reconstruction using $\zv \sim q_{\phi}$, and the third to the twelfth columns use $\zv$ from $\tilde q_{\eta,\phi}^{(t)}$ for $t=1,\ldots,10$, respectively, with fine-tuned step sizes. 
Overall, the reconstructions are good enough by $\zv \sim q_{\phi}$ and can be further improved by $\tilde q_{\eta,\phi}^{(t)}$ as $t$ increases.

\begin{figure}[!ht]
 \centering
 \begin{subfigure}[t]{0.48\textwidth}
 \centering
\includegraphics[width=0.8\linewidth]{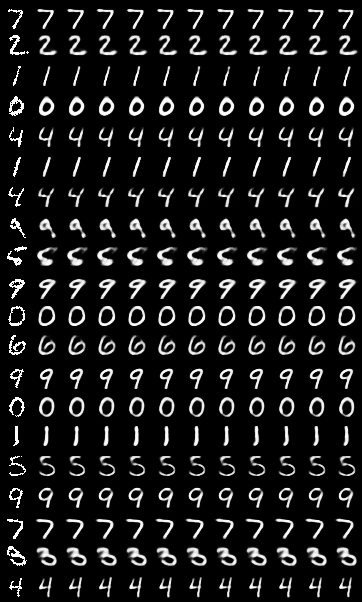}\vspace{-1mm} 
 \caption{\small Randomly selected.}\vspace{-1mm}
 \end{subfigure}%
 \hfill
 \begin{subfigure}[t]{0.48\textwidth}
 \centering
\includegraphics[width=0.8\linewidth]{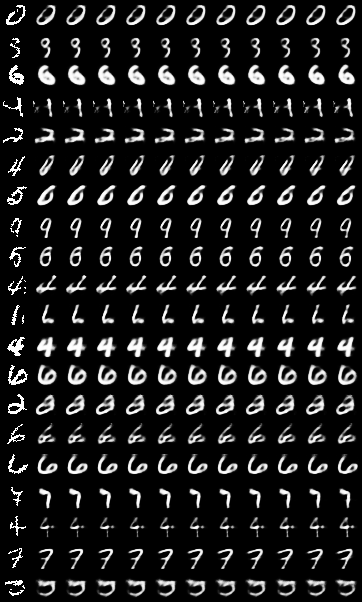}\vspace{-1mm} 
 \caption{\small Most improved.}\vspace{-1mm}
 \end{subfigure}%
\caption{\small VAE reconstructions of binarized MNIST testing images by MIVI ($\zv \in\mathbb{R}^{10}$).
 }\label{fig:mnist_test}
\end{figure}

\end{document}